\documentclass[11pt,letterpaper]{article}
\usepackage{neutralpreprint}
\usepackage{mathtools}
\usepackage{graphicx,booktabs,float,array}
\usepackage{placeins,needspace}
\usepackage{flafter}
\usepackage{algpseudocode}
\usepackage{tikz}
\usepackage{microtype}
\usepackage{pdflscape}
\usepackage{hyperref}
\usepackage{url}
\ifdefined\pdfinfoomitdate
\fi
\ifdefined\pdfminorversion
\fi

\theoremstyle{preprintplain}
\newtheorem{theorem}{Theorem}
\newtheorem{proposition}[theorem]{Proposition}
\newtheorem{lemma}[theorem]{Lemma}
\newcounter{example}
\newenvironment{example}[1]
  {\par\Needspace{4\baselineskip}\refstepcounter{example}\normalfont
   \pushQED{\qed}%
   \noindent\textbf{Example \theexample\ (#1).}\ \ignorespaces}
  {\popQED\par\ignorespacesafterend}
\newcounter{resultcondition}[theorem]
\renewcommand{\theresultcondition}{\roman{resultcondition}}
\newcommand{\resultcondition}{\refstepcounter{resultcondition}(\theresultcondition)~}
\floatstyle{ruled}
\newfloat{algorithm}{tbp}{loa}
\floatname{algorithm}{Algorithm}
\newcolumntype{L}[1]{>{\raggedright\arraybackslash}p{#1}}
\newcommand{\G}{\Gamma}
\newcommand{\1}{\mathbf 1}
\newcommand{\TV}{\operatorname{TV}}
\newcommand{\E}{\mathbb E}

\newcommand{\osc}{\operatorname{osc}}

\hypersetup{pdftitle={Regional summaries of Bayesian model uncertainty},
  pdfsubject={},pdfkeywords={Model uncertainty; Posterior summaries; Credible sets;
  Bayesian model averaging; Importance sampling; Adaptive search}}

\title{Regional summaries of Bayesian model uncertainty}
\author{\href{https://orcid.org/0009-0001-0493-237X}{Hanqing Li}\qquad
  Robert Nguyen\qquad
  \href{https://orcid.org/0000-0002-0905-2697}{Xuewen Lu}\\[4pt]
  {\small Department of Mathematics and Statistics, University of Calgary}\\[2pt]
  {\small\href{mailto:hanqing.li@ucalgary.ca}{hanqing.li@ucalgary.ca}\qquad
  \href{mailto:huy.nguyen4@ucalgary.ca}{huy.nguyen4@ucalgary.ca}\qquad
  \href{mailto:xlu@ucalgary.ca}{xlu@ucalgary.ca}}}
\hypersetup{pdfauthor={Hanqing Li, Robert Nguyen, and Xuewen Lu}}
\date{}

\AtBeginDocument{%
  \setlength{\textfloatsep}{8pt plus 2pt minus 2pt}%
  \setlength{\intextsep}{6pt plus 2pt minus 2pt}%
  \setlength{\floatsep}{6pt plus 2pt minus 2pt}%
}
\begin{document}
\raggedbottom
\maketitle

\begin{abstract}
Bayesian posterior uncertainty can be spread across many models that share the same substantive pattern. Lists of individual models can obscure this structure, while marginal summaries do not describe which configurations are jointly supported. We study how to represent joint model uncertainty through a compact collection of interpretable regions, each grouping models that share a specified pattern. We formulate this task as minimizing representation complexity subject to retaining a prescribed posterior mass within an analyst-supplied family. To address unknown regional probabilities, we develop Adaptive Regional Selection (ARS), which combines adaptive exploration and regional weight estimation with conservative and optimistic selection. These complementary problems identify a summary meeting the mass target and bound how much further its complexity could improve. Under valid regional bounds, we establish guarantees for retained mass, total-variation error, and suboptimality, together with conditions for certifying optimality. We also show that, under bounded within-region weight variation, sufficient sample counts for regional estimation need not grow with the number of represented models, and construct bounds that remain valid under adaptive sampling and stopping. Empirical studies illustrate regional summaries in Bayesian model averaging and Bayesian networks, examining retained mass, representation complexity, and certification.
\end{abstract}

\begingroup\small
\par
\endgroup

\section{Introduction}

Bayesian model uncertainty is often difficult to communicate because posterior support can be spread across many configurations. In variable selection, correlated predictors may play similar roles, so different supports can express much the same substantive explanation \citep{hoeting1999,ghosh2015collinearity}. A short list of probable models can omit substantial uncertainty, while a list retaining most of the mass can be long and repetitive. Marginal inclusion probabilities summarize variables separately but do not describe which combinations are jointly supported. We seek joint summaries that make these shared patterns visible while retaining a prescribed posterior mass.

The central idea is to describe models collectively through the patterns they share. Our primary application is sparse Bayesian model averaging (BMA), where each state is a predictor support. Consider ten supplied pairs of substitutable predictors. Selecting one predictor from each pair gives $2^{10}=1{,}024$ supports, all described by the same group-level rule. More generally, $H$ groups of $m$ predictors give $m^H$ such supports \citep{chipman1996,george2010,papaspiliopoulos2017}. One region can represent this joint rule while retaining uncertainty about which predictor is selected within each group. Its probability aggregates the constituent model probabilities, allowing unequal weights and dependence across groups. The analyst supplies the groups and their interpretation. The method selects which regions to retain.

Let $\Pi(\cdot\mid D)$ be a posterior on a large finite state space $\G$, and let $\mathcal A$ be an analyst-specified family of candidate retained sets, with representation complexity $C(A)$ for describing $A$. Given an omitted-mass tolerance $\epsilon\in[0,1)$, our objective is
\[
  \min_{A\in\mathcal A} C(A)
  \quad\text{subject to}\quad
  \Pi(A\mid D)\ge 1-\epsilon.
  \tag{P}\label{eq:problem}
\]
Since $\epsilon<1$, every feasible retained set is nonempty. A \emph{state} is one complete discrete configuration, a \emph{region} is a declared collection of states, and the \emph{retained set} $A$ is a union of regions. The application fixes a partition tree and additive terminal-region complexities, so complexity measures the size of the regional description. We use \emph{posterior compression} for this shorter representation and its associated truncation. It does not imply reduced search or prediction cost. Renormalizing the posterior within $A$ gives a total-variation distance from the original posterior equal to the omitted mass, controlling the change in posterior summaries and predictions.

The main difficulty is that a concise description does not make its posterior probability easy to obtain. A region can contain too many models for direct summation, and both its total weight and the overall normalizing constant may be unknown. Selection must therefore account for uncertainty about whether a summary meets the mass target and whether a simpler feasible summary exists.

We develop Adaptive Regional Selection (ARS) to address these two questions using regional weight bounds. The conservative selection problem identifies a union that can already be verified to meet the mass target, yielding a \emph{mass certificate}. The optimistic selection problem gives a lower bound on the complexity of any feasible union, including those using regions not yet exposed. The gap between the returned complexity and this lower bound gives a \emph{complexity certificate} for the improvement still possible. Tree expansion exposes candidate regions, while adaptive sampling tightens their weight bounds. The guarantees depend on valid regional bounds, with their mathematical, sampling, and numerical scope specified in Section~\ref{sec:guarantees}.

Our contributions are:
\begin{enumerate}
\item \emph{Regional selection with mass and complexity guarantees.} We develop ARS and prove bounds on the omitted mass and suboptimality of its verified regional summaries. The omitted-mass bound controls total-variation error. A decision-margin result identifies sufficient precision for certifying the optimum once all terminal regions are available.
\item \emph{An explanation of the grouping advantage.} We prove that, for a region with sufficient posterior mass and bounded variation in state weights, an accurate individual-model list must grow proportionally to the region's size. At fixed accuracy and weight variation, the sufficient sample count for estimating the region's total weight by uniform sampling is independent of that size.
\item \emph{Regional bounds valid under adaptive computation.} We construct regional mass bounds using bounded importance sampling and prespecified checkpoints, with simultaneous coverage under adaptive allocation and stopping. For Gaussian BMA, computable envelopes bound the sampling weights. Parent--child mass identities provide further tightening of the stored intervals.
\end{enumerate}
Experiments and applications in BMA and Bayesian networks examine the resulting joint summaries. The BMA comparisons use the same regional family for ARS and Monte Carlo. ARS more consistently retains sufficient mass and provides additional conditional certification, while complexity performance varies by design. Network experiments isolate the representation benefit of grouping graph orientations using supplied masses. Exact references distinguish feasibility, attained complexity, and certification, with all unsuccessful searches retained.

\subsection{Related work}
\label{sec:related-work}

Closely related work seeks sets with a prescribed posterior probability. Minimum-volume sets, highest-density regions, and discrete smallest covering regions minimize set size subject to probability coverage \citep{scott2006mvs,hyndman1996,oneill2022}. For continuous posteriors, \citet{lee2021tree} use density estimation trees to estimate highest-density credible sets. Here complexity measures a structured description of many discrete states. The closest variable-selection approach, Cartesian credible sets \citep{griffin2025}, already targets interpretable joint model uncertainty. Our focus is mass and complexity guarantees under uncertain regional weights. Table~\ref{tab:regional-comparison} separates representation from computation, since Cartesian credible sets need not assume independence. In phylogenetics, \citet{klawitter2026credible} construct compact credible sets of tree topologies by pruning a fitted conditional clade distribution and assessing retained mass under that distribution.

\begin{table}[H]
\centering
\caption{Comparison of Cartesian credible sets and ARS.}
\label{tab:regional-comparison}
\renewcommand{\arraystretch}{1.08}
\begin{tabular*}{\textwidth}{@{\extracolsep{\fill}}L{.17\textwidth}L{.36\textwidth}L{\dimexpr.47\textwidth-4\tabcolsep\relax}@{}}
\toprule
Aspect & Cartesian credible sets & ARS\\
\midrule
Representation & Cartesian combinations of allowed submodels within predictor blocks. & Unions of terminal regions in a supplied partition.\\
Structural choice & Choose the partition by balancing fidelity and interpretability. & Fix the partition and additive regional complexities before search.\\
Mass calculation & For nonfactorizing posteriors, computation uses a block-factorized approximation. & Bound regional weight sums under the original posterior.\\
Selection aim & Construct interpretable joint credible sets of models. & Minimize declared complexity at a mass target, with conditional mass and complexity bounds.\\
\bottomrule
\end{tabular*}
\end{table}

Other structured summaries target different events. Fine-mapping credible sets and SuSiE localize effects \citep{wtccc2012,wang2020susie}, while BLiP selects signal-location regions under a Bayesian false-positive constraint \citep{spector2025blip}. Simultaneous changepoint sets cover candidate locations \citep{siems2019simultaneous}, and predictive subset families collect models with near-optimal predictive performance \citep{kowal2022subsets}. ARS instead selects a joint event comprising complete configurations.

Occam's window retains models through relative posterior support \citep{madigan1994}, and Rashomon Partition Sets enumerate models above a posterior-density threshold \citep{venkateswaran2024rashomon}. Bayesian Adaptive Sampling, adaptive Markov chain Monte Carlo (MCMC), and importance tempering explore variable-selection spaces \citep{clyde2011bas,griffin2021asi,zanella2019tgs}. ARS uses uncertainty in regional totals to guide exploration toward a retained-mass constraint and a representation-complexity objective.

Choosing where to sample connects ARS to combinatorial pure exploration, which allocates samples where uncertainty affects a discrete decision \citep{chen2017bestset,cao2019disagreement}. In ARS, complexities are known but feasibility depends on unknown regional masses and an unknown normalizing constant. Simultaneous sequential inference preserves validity under adaptive allocation and stopping \citep{kaufmann2021mixture}. We obtain this control by distributing a finite failure allowance across regional sampling runs and prespecified checkpoints. Combining these intervals with optimistic and conservative selection gives bounds on omitted mass and possible reduction in complexity throughout the search.

\subsection{Organization of the paper}
\label{sec:organization}

\looseness=-1
Section~\ref{sec:target} defines the problem, and Section~\ref{sec:search} develops ARS and its guarantees. Sections~\ref{sec:experiments}--\ref{sec:discussion} present simulations, Tecator and Sachs applications, and discussion. Appendices~\ref{app:proofs}--\ref{app:experiments} provide proofs, constructions, and empirical details.

\section{Joint posterior summaries through structured partitions}
\label{sec:target}
\subsection{Posterior target and truncation error}
\label{sec:target-sets}

Let $\G$ be a finite space of admissible states and let $s(\gamma)>0$ be an unnormalized posterior weight given data $D$. Define $Z(E)=\sum_{\gamma\in E}s(\gamma)$ for $E\subseteq\G$ and $\Pi(\gamma\mid D)=s(\gamma)/Z(\G)$. We call $Z(E)$ the partial normalizing constant, or \emph{regional constant} when $E$ is a region. States need not be binary, and $\G$ contains only admissible configurations, such as acyclic graphs in Bayesian-network applications. In our Gaussian BMA applications, $\G=2^{\{1,\ldots,p\}}$, a state $\gamma$ is the set of selected predictor indices, and $s(\gamma)$ is a marginal likelihood times a support prior \citep{hoeting1999,papaspiliopoulos2017}.

For any nonempty $A\subseteq\G$, positivity of the weights ensures $\Pi(A\mid D)>0$, and the ideal truncated posterior is
\[
  \Pi^A(E\mid D)=\frac{\Pi(E\cap A\mid D)}{\Pi(A\mid D)}.
  \tag{1}\label{eq:truncated-posterior}
\]
This is the original posterior renormalized within $A$. Writing $\TV$ for total-variation distance, direct calculation gives
\[
\TV(\Pi^A,\Pi)=1-\Pi(A\mid D).
\tag{2}\label{eq:truncation-errors}
\]
Thus the mass constraint in Problem~\eqref{eq:problem} directly bounds the error introduced by truncation. Applying a common posterior-to-predictive mapping cannot increase TV distance, so the same $\epsilon$ bounds predictive TV.

For a bounded function $f:\G\to\mathbb R$, the quantity $\E_\Pi f=\sum_{\gamma\in\G}f(\gamma)\Pi(\gamma\mid D)$ is its posterior mean. For example, in BMA, $f(\gamma)=\lvert\gamma\rvert$ counts the selected predictors. Under the mass constraint,
$\lvert \E_{\Pi^A}f-\E_\Pi f\rvert \le\epsilon\osc_\G f$,
where $\osc_\G f=\max_\G f-\min_\G f$. Thus event probabilities change by at most $\epsilon$ and expected model size by at most $p\epsilon$ for $p$ predictors. Appendix~\ref{app:truncation-proof} derives these consequences.

\subsection{Structured partitions and representation complexity}
\label{sec:regions-cost}

A single region can represent several models, and the returned summary selects a union of such regions. Formally, a \emph{state} is one complete configuration $\gamma\in\G$, such as a predictor support in BMA. A \emph{region} is a nonempty subset of states. A \emph{partition} of $\G$ consists of pairwise disjoint regions whose union is $\G$.

The analyst supplies a fixed finite rooted \emph{partition tree}, with node set $\mathcal V$ and root $r$. Here $v$ is a node label, not necessarily an integer. Each node $v\in\mathcal V$ represents a region $R_v\subseteq\G$, with $R_r=\G$, and has regional constant $Z_v:=Z(R_v)=\sum_{\gamma\in R_v}s(\gamma)$. An \emph{internal node} has children $\mathcal H(v)$ whose regions partition its region. A \emph{leaf} has no children, and the leaf set $\mathcal L\subseteq\mathcal V$ defines the \emph{terminal partition} $\{R_v:v\in\mathcal L\}$. Write $\mathcal L(v)$ for the leaves in the subtree rooted at $v$, with $\mathcal L(v)=\{v\}$ at a leaf. Terminal means that splitting stops there, even if a leaf contains many states. Each child region is contained in its parent region, while regions belonging to different branches are disjoint.

\looseness=-1
For BMA, ordered groups $G_1,\ldots,G_H$ partition $\{1,\ldots,p\}$. Group sizes may differ and may exceed two. A depth-$d$ node has path label $v=(a_1,\ldots,a_d)$, where $a_h\in\{0,1,2\}$ means zero, one, or at least two selected predictors from group $h$. Its region is $R_v=\{\gamma\in\G:\min\{2,|\gamma\cap G_h|\}=a_h,\ 1\le h\le d\}$. Root $r=()$ imposes no restrictions. At depth one, $a$ abbreviates $(a)$. Children append admissible categories, omitting empty regions, and descendant labels extend that path. Groups $h>d$ remain unrestricted. The supplied tree fixes the terminal nodes. Labels stay fixed during search, and depth $d$ differs from update $k$.

The label $a_h=2$, written $2+$ in descriptions, combines selections of two or more predictors from $G_h$ without specifying their identities. Each model contributes its own weight $s(\gamma)$ to the regional total $Z_v$. Retaining a region includes all its constituent models in the truncated posterior, preserving their relative posterior weights.

The representation selects whole terminal regions. Let $\mathcal A$ be the family of all their unions, including the empty union. A selection $J\subseteq\mathcal L$ contains leaf indices, whereas its \emph{retained set} $A_J$ contains the states in their regions. Assigning each leaf a nonnegative representation complexity $c_v$ gives
\[
  A_J=\bigcup_{v\in J}R_v,\qquad
  \mathcal A=\{A_J:J\subseteq\mathcal L\},\qquad
  C(A_J)=\sum_{v\in J}c_v.
  \tag{3}\label{eq:candidate-class}
\]
Problem~\eqref{eq:problem} selects a member of $\mathcal A$ meeting the mass target at minimum complexity $C^*_\epsilon$. Here complexity means representation complexity, excluding search computation. With unit weights $c_v=1$, $C(A_J)=\lvert J\rvert$ counts retained regions. Other weights allow differences between regional descriptions.

\begin{figure}[!tb]
  \centering
\begingroup
\definecolor{treeblue}{RGB}{25,100,150}
\definecolor{treeink}{RGB}{55,65,75}
\begin{tikzpicture}[
  x=1cm,y=1cm,
  font=\fontsize{8.5}{10}\selectfont,
  every node/.style={text=treeink},
  region/.style={draw=treeink!55,fill=white,rounded corners=2pt,
    line width=.5pt,align=center,inner sep=3pt},
  leaf/.style={region,text width=1.20cm,inner sep=2pt,minimum height=.85cm},
  selected/.style={draw=treeblue,fill=treeblue!24,text=treeblue,
    line width=.85pt,double=white,double distance=.65pt},
  level/.style={anchor=west,align=left,font=\bfseries\fontsize{8.5}{10}\selectfont},
  edge/.style={draw=treeink!45,line width=.6pt},
  detail/.style={draw=treeink!65,dashed,line width=.65pt}
]
\node[anchor=south] at (9.28,.92)
  {Group A: $x_1,x_2$ \qquad Group B: $x_3,x_4$};
\node[anchor=south] at (9.28,.54)
  {Leaf label = (count from A, count from B)};

\node[level] at (0,0) {Root\\all models};
\node[region,fill=treeink!3,minimum width=2.65cm,minimum height=.72cm] (root)
  at (9.28,0) {$R_r=\Gamma$\\16 models};

\node[level] at (0,-1.35) {Internal regions\\count from A};
\node[region,minimum width=2.05cm,minimum height=.75cm] (g0)
  at (4.52,-1.35) {$R_0$: 0 from A\\4 models};
\node[region,minimum width=2.05cm,minimum height=.75cm] (g1)
  at (9.28,-1.35) {$R_1$: 1 from A\\8 models};
\node[region,minimum width=2.05cm,minimum height=.75cm] (g2)
  at (14.04,-1.35) {$R_2$: 2 from A\\4 models};
\foreach \g in {g0,g1,g2} \draw[edge] (root.south) -- (\g.north);

\node[level] at (0,-2.85) {Leaves\\terminal\\partition};
\node[leaf] (r00) at (3.02,-2.85) {$R_{(0,0)}$\\1 model};
\node[leaf,selected] (r01) at (4.52,-2.85) {$R_{(0,1)}$\\2 models};
\node[leaf] (r02) at (6.02,-2.85) {$R_{(0,2)}$\\1 model};
\node[leaf] (r10) at (7.78,-2.85) {$R_{(1,0)}$\\2 models};
\node[leaf,selected] (r11) at (9.28,-2.85) {$R_{(1,1)}$\\4 models};
\node[leaf] (r12) at (10.78,-2.85) {$R_{(1,2)}$\\2 models};
\node[leaf] (r20) at (12.54,-2.85) {$R_{(2,0)}$\\1 model};
\node[leaf] (r21) at (14.04,-2.85) {$R_{(2,1)}$\\2 models};
\node[leaf] (r22) at (15.54,-2.85) {$R_{(2,2)}$\\1 model};
\foreach \r in {r00,r01,r02} \draw[edge] (g0.south) -- (\r.north);
\foreach \r in {r10,r11,r12} \draw[edge] (g1.south) -- (\r.north);
\foreach \r in {r20,r21,r22} \draw[edge] (g2.south) -- (\r.north);

\draw[detail] (r11.south) -- (9.28,-3.62) -- (12.15,-3.62)
  -- (12.15,-4.02);
\node[region,draw=treeblue,fill=treeblue!12,
  minimum width=7.30cm,minimum height=1.50cm,anchor=north]
  at (3.80,-4.02) {};
\node[anchor=north,text=treeblue,font=\bfseries\fontsize{8.5}{10}\selectfont]
  at (3.80,-4.17) {One candidate retained set};
\node[text=treeblue] at (3.80,-4.80)
  {$A_J=R_{(0,1)}\cup R_{(1,1)}$};
\node[text=treeblue] at (3.80,-5.20) {$A_J\in\mathcal A$: 2 regions, 6 models};

\node[region,fill=treeink!2,
  minimum width=7.75cm,minimum height=1.50cm,anchor=north]
  at (12.15,-4.02) {};
\node[anchor=north,font=\bfseries\fontsize{8.5}{10}\selectfont]
  at (12.15,-4.17) {States in $R_{(1,1)}$};
\node[align=center] at (10.55,-4.96) {$\{1,3\}$\\[4pt]$\{2,3\}$};
\node[align=center] at (13.75,-4.96) {$\{1,4\}$\\[4pt]$\{2,4\}$};
\end{tikzpicture}
\endgroup
  \caption{A partition tree for two predictor groups.}
  \label{fig:partition-tree}
\end{figure}

\emph{Choosing regions and complexities.} Regions should express useful joint alternatives, such as substitution among predictor groups or common network adjacencies. Choose complexities to reflect the descriptions a reader will use, and report the grouping, terminal rule, and complexity rule with the output. Grouping defines the permitted summaries without imposing posterior independence. A coarser family may be shorter to describe but less informative. Changing the family or the complexity rule changes the optimization problem, so a complexity certificate is relative to those choices.

\begin{example}{Two predictor groups}
\label{ex:two-predictor-groups}
Figure~\ref{fig:partition-tree} illustrates the construction with Group A $G_1=\{1,2\}$ and Group B $G_2=\{3,4\}$. The state space contains all 16 supports, including the empty support. Root $r$ represents all of them. Its children $a\in\{0,1,2\}$ specify the category from Group A, and their children $(a,b)\in\mathcal L=\{0,1,2\}^2$ also specify the category from Group B. Here category two means exactly two because each group is a pair, and $R_a=\bigcup_{b=0}^2R_{(a,b)}$.

For example, $\gamma=\{1,3\}$ selects predictors $x_1,x_3$ and is one state in $R_{(1,1)}$, which contains the four supports shown in the inset. The inset lists states within that region, not further children. Its regional constant $Z_{(1,1)}$ sums their four weights, and its posterior probability is $Z_{(1,1)}/Z(\G)$. Four of the 16 states need not carry one-quarter of the posterior mass, because state weights may differ. The other blue region is $R_{(0,1)}=\{\{3\},\{4\}\}$. Thus $J=\{(0,1),(1,1)\}$ retains $A_J=R_{(0,1)}\cup R_{(1,1)}$, containing six states, at complexity $c_{(0,1)}+c_{(1,1)}$. With unit leaf complexities, $C(A_J)=2$. It is one of the $2^9-1=511$ nonempty unions in $\mathcal A$. Its weights determine whether it meets the mass target.
\end{example}

Regional mass measures posterior support for a declared pattern. Its interpretation depends on the supplied groups and categories, the constituent states, and their prior weights.

Formally, $\mathcal A$ is the sigma algebra generated by the terminal regions, which are its atoms, or smallest nonempty events. Selection treats each leaf as one unit, while the conditional posterior in \eqref{eq:truncated-posterior} remains defined on $2^{\G}$, the collection of all subsets of the state space, and preserves the relative probabilities of retained states. A finer terminal partition enlarges the candidate family, which equals $2^{\G}$ exactly when every leaf is a singleton. Exploring the supplied tree only exposes its fixed family. Internal nodes organize descendant leaves and are not entries in the returned representation.

\FloatBarrier
\subsection{Why grouping helps}
\label{sec:structured-economy}

The following result identifies when a single region can represent many models and its total posterior weight can be estimated without evaluating each model. Let $R_\star$ contain $M=\lvert R_\star\rvert $ states. Here $\xi$ bounds its omitted posterior mass and $\rho$ measures variation in the posterior weights themselves.

\begin{proposition}[Grouping advantage]
\label{prop:economy}
Suppose \resultcondition $\Pi(R_\star\mid D)\ge1-\xi$ for some $\xi\in[0,1)$; and \resultcondition $\osc_{\gamma\in R_\star}\log s(\gamma)\le\rho$. Then the following conclusions hold.
\begin{enumerate}[label=(\alph*),leftmargin=*,itemsep=2pt,topsep=3pt]
\item For $\eta\in[0,1-\xi)$, any nonempty subset $B\subseteq\G$ with $\TV(\Pi^B,\Pi)\le\eta$ satisfies
$\lvert B\rvert \ge e^{-\rho}(1-\eta-\xi)M$.
\item For independent draws $\gamma_i\sim\operatorname{Unif}(R_\star)$, $i=1,\ldots,N$, let $\widehat Z_N=M N^{-1}\sum_i s(\gamma_i)$. If $\rho>0$, then for every $u>0$,
\[
\Pr\bigl(\bigl\lvert \widehat Z_N/Z(R_\star)-1\bigr\rvert >u\bigr)
\le2\exp\left\{-\frac{2Nu^2}{(e^\rho-1)^2}\right\}.
\]
If $\rho=0$, then $s$ is constant on $R_\star$, so $Z(R_\star)=Ms(\gamma)$ for every $\gamma\in R_\star$. Consequently, $\widehat Z_N=Z(R_\star)$ almost surely for every $N\ge1$.
\end{enumerate}
\end{proposition}

Appendix~\ref{app:grouping-proof} gives the proof. For fixed $\rho$, the number of states in an accurate retained set must grow at least proportionally to $M$, whereas the sufficient evaluation count for a regional constant is independent of $M$. Under exact symmetry, one representative gives $Z(R_\star)=Ms(\gamma)$. Realizing this advantage requires a selectable region $R_\star$, sampling access without enumeration, and a representation complexity that treats the region as a compact unit. Section~\ref{sec:search} develops the search and mass estimation needed to select such regions.

\section{Adaptive regional selection}
\label{sec:search}

ARS seeks a union of terminal regions that meets a posterior-mass target with the smallest representation complexity. The application supplies the posterior-weight evaluator, fixed partition tree, and regional complexity rule. For BMA, these come from the response and design data, model specification, predictor groups, and complexity choice described in Section~\ref{sec:ars-bma}. The tree defines the permitted summaries and remains fixed throughout the search.

The main computational difficulty is obtaining the regional mass information needed for selection. ARS therefore separates two tasks: acquiring mass information and solving selection problems with that information. A feasible conservative selection gives a verified union, while an optimistic selection bounds how much its complexity could still improve. Section~\ref{sec:ars-algorithm} defines these problems, and Section~\ref{sec:adaptive-updates} gives the two-phase procedure. Section~\ref{sec:ars-models} gives the Gaussian regional calculations, and Section~\ref{sec:guarantees} states the guarantees and their scope.

\subsection{Objective and selection bounds}
\label{sec:ars-algorithm}
\label{sec:adaptive-selection}

For each terminal node $v\in\mathcal L$, let $x_v\in\{0,1\}$ be the binary decision variable indicating whether region $R_v$ is retained. The terminal regions partition $\G$, so Problem~\eqref{eq:problem} is
\[
 C^*_\epsilon
 =\min_{x\in\{0,1\}^{\mathcal L}}\sum_{v\in\mathcal L}c_vx_v
 \quad\text{subject to}\quad
 \sum_{v\in\mathcal L}Z_vx_v
 \ge(1-\epsilon)\sum_{v\in\mathcal L}Z_v.
\]
This is a discrete optimization problem: choose terminal regions of minimum total complexity while retaining at least $1-\epsilon$ posterior mass. The regional complexities $c_v$ and evaluator $s$ are supplied, but computing every coefficient $Z_v$ may be expensive. ARS uses bounds on these fixed coefficients in the two selection problems below. Expansion and sampling decide which missing mass information to acquire.

At selection update $k\ge0$, with $k=0$ at initialization, let $\mathcal V^{(k)}$ contain all nodes recorded so far. Here an update is a completed change to the search state. Current search states and selection outputs use superscripts $(k)$, while saved outputs use subscripts $k$. Other subscripts identify nodes and local indices. Three objects describe the available mass information:
\begin{itemize}[leftmargin=*,itemsep=2pt,topsep=3pt]
\item A \emph{frontier} $\mathcal F^{(k)}\subseteq\mathcal V^{(k)}$ of node labels whose regions partition $\G$, possibly at different depths. Expanding a node replaces it by its children. In Example~\ref{ex:two-predictor-groups}, expanding the root and then node $0$ gives \mbox{$\mathcal F^{(k)}=\{(0,0),(0,1),(0,2),1,2\}$}. The recorded set $\mathcal V^{(k)}$ also keeps ancestors $r$ and $0$. Internal frontier nodes stand for terminal descendants not yet exposed.
\item An interval record $B^{(k)}=(B_v^{(k)})_{v\in\mathcal V^{(k)}}$, where $B_v^{(k)}=[\underline Z_v^{(k)},\overline Z_v^{(k)}]$ bounds the fixed total $Z_v$. This is a collection of intervals indexed by recorded nodes. Ancestor intervals are retained for mass-conservation updates.
\item \emph{Eligible leaves} $\mathcal E^{(k)}\subseteq\mathcal F^{(k)}\cap\mathcal L$ whose initial intervals have been adopted for selection. Adoption means making a leaf available to the conservative selection problem. Its interval may be exact, sampled, or deterministic, and need not be narrow. Only terminal regions enter the returned representation.
\end{itemize}
For $S\subseteq\mathcal F^{(k)}$, write $\underline Z^{(k)}(S)=\sum_{v\in S}\underline Z_v^{(k)}$ and $\overline Z^{(k)}(S)=\sum_{v\in S}\overline Z_v^{(k)}$. A terminal selection $J$ meets the exact mass constraint precisely when
\[
 \epsilon Z(A_J)-(1-\epsilon)Z(A_J^c)\ge0.
\]
The conservative test uses a lower bound on retained mass and an upper bound on omitted mass, so a passing terminal union is verified. The optimistic test reverses these endpoints and permits internal frontier nodes to stand for selections from their unexposed descendants. An optimistic selection guides further work but need not itself be a permitted terminal union. Let $\underline c_v\ge0$ bound below the complexity of every nonempty selection of descendant leaves below $v$. For nonnegative additive complexities, one may supply valid bounds or compute the tight choice $\underline c_v=\min_{w\in\mathcal L(v)}c_w$ using the subtree leaf set defined in Section~\ref{sec:regions-cost}.

With valid intervals and fixed descendant-complexity bounds, the conservative and optimistic selection problems are
\[
\overline C^{(k)}
=\min_{\substack{\varnothing\ne S\subseteq\mathcal E^{(k)}\\
 (1-\epsilon)\overline Z^{(k)}(\mathcal F^{(k)}\setminus S)
 \le\epsilon\underline Z^{(k)}(S)}}\sum_{v\in S}c_v,
\qquad
\underline C^{(k)}
=\min_{\substack{\varnothing\ne S\subseteq\mathcal F^{(k)}\\
 (1-\epsilon)\underline Z^{(k)}(\mathcal F^{(k)}\setminus S)
 \le\epsilon\overline Z^{(k)}(S)}}\sum_{v\in S}\underline c_v.
\tag{4}\label{eq:frontier-programs}
\]
A conservative optimizer $J^{(k)}\subseteq\mathcal E^{(k)}$ retains $A_{J^{(k)}}=\bigcup_{v\in J^{(k)}}R_v$. Its conservative omitted-mass bound is $\zeta^{(k)}:=\overline Z^{(k)}(\mathcal F^{(k)}\setminus J^{(k)})/[\underline Z^{(k)}(J^{(k)})+\overline Z^{(k)}(\mathcal F^{(k)}\setminus J^{(k)})]$, with $\zeta^{(k)}=0$ when $J^{(k)}=\mathcal F^{(k)}$, even if the denominator is zero.

For any feasible terminal union $A$, the set $\mathcal S^{(k)}(A)=\{v\in\mathcal F^{(k)}:R_v\cap A\ne\varnothing\}$ passes the optimistic test and satisfies $\sum_{v\in \mathcal S^{(k)}(A)}\underline c_v\le C(A)$. Thus $\underline C^{(k)}\le C^*_\epsilon\le\overline C^{(k)}$ whenever the conservative problem is feasible.

An infeasible selection problem has minimum $+\infty$, with $(J^{(k)},\zeta^{(k)})=(\varnothing,+\infty)$ for an infeasible conservative problem. Write $S^{(k)}$ for an optimistic optimizer. Both problems are covering knapsacks, solved by dynamic programming for positive integer complexity weights. Appendix~\ref{app:optimization} gives the recurrence and numerical checks.

\subsection{The algorithm and stopping rule}
\label{sec:adaptive-updates}

Algorithm~\ref{alg:workflow} specifies the two-phase procedure used in the main BMA comparison. Phase 1 expands the tree and evaluates exposed leaves. Phase 2 holds the frontier fixed and refines mass intervals for leaves deferred by Phase 1. Both phases call the same selection and bound-update routine. The subscript $k$ marks saved selections and bounds. Section~\ref{sec:ars-models} supplies the Gaussian regional routines.

\begin{algorithm}[!t]
\caption{ARS: partition exploration followed by regional mass refinement}
\label{alg:workflow}
\label{alg:detailed-workflow}
\small
\algrenewcommand\algorithmicrequire{\textbf{Input:}}
\algrenewcommand\algorithmicensure{\textbf{Output:}}
\begin{algorithmic}[1]
\setlength{\baselineskip}{11.5pt}
\Require Posterior-weight evaluator $s$; fixed regional tree with leaf complexities $c_v$; retained-mass target $1-\epsilon$.
\Statex \textbf{Settings:} Accuracy tolerances and computational budgets, defined in Section~\ref{sec:adaptive-updates}.
\Ensure Selected region labels $J_k$, omitted-mass bound $\zeta_k$, complexity bounds $[\underline C_k,\overline C_k]$, and stopping status.
\State $k\gets-1$; $\mathcal C_{-1}\gets(\varnothing,0,+\infty,+\infty)$; $(B_r^{(0)},\mathit{status})\gets\Call{AnchorBounds}{r}$
\State \textbf{if} $\mathit{status}\ne\textsc{ok}$ \textbf{then return} $(\mathcal C_{-1},\textsc{failure})$
\State $\mathcal V^{(0)}\gets\{r\}$; $\mathcal F^{(0)}\gets\{r\}$; $\mathcal D^{(0)}\gets\varnothing$; $\mathcal E^{(0)}\gets\{v\in\{r\}\cap\mathcal L:\underline Z_v^{(0)}=\overline Z_v^{(0)}\}$
\State $\mathcal U^{(0)}\gets\mathcal F^{(0)}\setminus\mathcal E^{(0)}$; $(B^{(0)},S^{(0)},\mathcal C_0,\mathit{status})\gets\Call{Update}{0,\mathcal F^{(0)},\mathcal E^{(0)},B^{(0)},\mathcal C_{-1}}$
\State \textbf{if} $\mathit{status}\ne\textsc{ok}$ \textbf{then return} $(\mathcal C_{-1},\textsc{failure})$; $k\gets0$
\Statex \textbf{Phase 1: partition exploration}
\While{$\Call{BaseOpen}{\mathcal U^{(k)},\mathcal B}$ and $\neg\Call{Done}{\mathcal C_k}$}
  \State $v\gets\Call{Priority}{\mathcal U^{(k)}\cap S^{(k)}\ \textbf{if }\mathcal U^{(k)}\cap S^{(k)}\ne\varnothing\ \textbf{else }\mathcal U^{(k)}}$
  \If{$v\notin\mathcal L$}
    \State $((B_w^{(k+1)})_{w\in\mathcal H(v)},\mathit{status})\gets\Call{AnchorBounds}{\mathcal H(v)}$
    \State \textbf{if} $\mathit{status}=\textsc{failure}$ \textbf{then return} $(\mathcal C_k,\mathit{status})$; \textbf{if} $\mathit{status}=\textsc{incomplete}$ \textbf{then break}
    \State $\mathcal V^{(k+1)}\gets\mathcal V^{(k)}\cup\mathcal H(v)$; $\mathcal F^{(k+1)}\gets(\mathcal F^{(k)}\setminus\{v\})\cup\mathcal H(v)$
    \State $\mathcal E^{(k+1)}\gets\mathcal E^{(k)}\cup\{w\in\mathcal H(v)\cap\mathcal L:\underline Z_w^{(k+1)}=\overline Z_w^{(k+1)}\}$
  \Else
    \State $(B_v^{(k+1)},\mathit{blocked},\mathit{status})\gets\Call{CompleteLeaf}{v,B_v^{(k)},\delta/(2K)}$
    \State \textbf{if} $\mathit{status}=\textsc{failure}$ \textbf{then return} $(\mathcal C_k,\mathit{status})$; \textbf{if} $\mathit{status}=\textsc{incomplete}$ \textbf{then break}
    \State $\mathcal E^{(k+1)}\gets\mathcal E^{(k)}\cup\{v\}$; \textbf{if} $\mathit{blocked}$ \textbf{then} $\mathcal D^{(k+1)}\gets\mathcal D^{(k)}\cup\{v\}$
  \EndIf
  \State $\mathcal U^{(k+1)}\gets\mathcal F^{(k+1)}\setminus(\mathcal E^{(k+1)}\cup\mathcal D^{(k+1)})$
  \State $(B^{(k+1)},S^{(k+1)},\mathcal C_{k+1},\mathit{status})\gets\Call{Update}{k+1,\mathcal F^{(k+1)},\mathcal E^{(k+1)},B^{(k+1)},\mathcal C_k}$
  \State \textbf{if} $\mathit{status}=\textsc{failure}$ \textbf{then return} $(\mathcal C_k,\mathit{status})$; \textbf{if} $\mathit{status}=\textsc{incomplete}$ \textbf{then break}
  \State $k\gets k+1$
\EndWhile
\Statex \textbf{Phase 2: regional mass refinement}
\State $k_0\gets k$; $\mathcal R\gets\mathcal E^{(k_0)}\cap\mathcal D^{(k_0)}$; $\alpha\gets\delta/(2KT)$
\State $(t_v^{(k_0)},\mathcal W_{v,0})\gets(0,())$ for all $v\in\mathcal R$; $\mathcal Q^{(k)}\gets\Call{Affordable}{\mathcal R,(t_v^{(k)})_{v\in\mathcal R},\mathcal B}$
\While{$\mathcal Q^{(k)}\ne\varnothing$ and $\neg\Call{Done}{\mathcal C_k}$}
  \State $v\gets\Call{Pick}{\mathcal Q^{(k)},S^{(k)},B^{(k)},(t_v^{(k)})_{v\in\mathcal R}}$; $t\gets t_v^{(k)}+1$
  \If{$t_v^{(k)}=0$}
    \State $(q_v,U_v,\mathit{status})\gets\Call{PrepareEnvelope}{s,R_v,\mathcal B}$
    \State \textbf{if} $\mathit{status}=\textsc{failure}$ \textbf{then return} $(\mathcal C_k,\mathit{status})$; \textbf{if} $\mathit{status}=\textsc{incomplete}$ \textbf{then break}
  \EndIf
  \State $(w,\mathit{status})\gets\Call{SampleWeights}{s,q_v,U_v,N_t-N_{t-1},\mathcal B}$
  \State \textbf{if} $\mathit{status}=\textsc{failure}$ \textbf{then return} $(\mathcal C_k,\mathit{status})$; \textbf{if} $\mathit{status}=\textsc{incomplete}$ \textbf{then break}
  \State $\mathcal W_{v,t}\gets\mathcal W_{v,t-1}\mathbin\Vert w$; $B_v^{(k+1)}\gets B_v^{(k)}\cap\mathcal I_{v,t}\cap[0,U_v]$
  \State \textbf{if} $B_v^{(k+1)}=\varnothing$ \textbf{then return} $(\mathcal C_k,\textsc{failure})$
  \State $t_v^{(k+1)}\gets t$; $(B^{(k+1)},S^{(k+1)},\mathcal C_{k+1},\mathit{status})\gets\Call{Update}{k+1,\mathcal F^{(k+1)},\mathcal E^{(k+1)},B^{(k+1)},\mathcal C_k}$
  \State \textbf{if} $\mathit{status}=\textsc{failure}$ \textbf{then return} $(\mathcal C_k,\mathit{status})$; \textbf{if} $\mathit{status}=\textsc{incomplete}$ \textbf{then break}
  \State $k\gets k+1$; $\mathcal Q^{(k)}\gets\Call{Affordable}{\mathcal R,(t_v^{(k)})_{v\in\mathcal R},\mathcal B}$
\EndWhile
\State \textbf{if} $\Call{Done}{\mathcal C_k}$ \textbf{then return} $(\mathcal C_k,\textsc{success})$
\State \textbf{if} $J_k=\varnothing$ \textbf{then return} $(\mathcal C_k,\textsc{no-verified-union})$
\State \Return $(\mathcal C_k,\textsc{unresolved-gap})$
\end{algorithmic}
\end{algorithm}

Data and model give the evaluator $s$; the region specification gives the fixed tree and leaf complexities $(c_v)_{v\in\mathcal L}$. Budgets $\mathcal B$ comprise shared score and time limits and an action limit for Phase 1. Both phases use mass tolerance $\epsilon$, complexity-gap tolerance $\tau_C\ge0$, and failure allowance $\delta\in(0,1)$. Fix a node-count bound $K\ge\lvert\mathcal V\rvert$ and an initial relative-error target $\eta_0\in(0,1)$.

The remaining fixed settings are the nonnegative descendant bounds $(\underline c_v)_{v\in\mathcal V}$, leaf-evaluation limits, and checkpoints $(N_t)_{t=0}^T$, described in Sections~\ref{sec:ars-algorithm} and~\ref{sec:ars-exploration}--\ref{sec:ars-refinement}. Proposal, anchor and bound constructions, and envelope sampling are model routines, built in for BMA (Section~\ref{sec:ars-models}). These settings and regional constructions are fixed before search. Experimental variants are stated alongside their designs.

The \emph{incumbent} $\mathcal C_k=(J_k,\underline C_k,\overline C_k,\zeta_k)$ is the saved output record. Its entries are the selected region labels, lower and upper bounds on minimum complexity, and the omitted-mass bound, respectively. If $J_k\ne\varnothing$, the retained union is $A_{J_k}=\bigcup_{v\in J_k}R_v$, with complexity $\overline C_k=C(A_{J_k})$. An empty selection means that no verified union is available. Algorithm~\ref{alg:workflow} returns these saved quantities at termination.

Start from $\mathcal C_{-1}=(\varnothing,0,+\infty,+\infty)$ and an initial interval for the root. Set $\mathcal E^{(0)}=\{r\}$ if the root is a leaf with exact bounds, and $\mathcal E^{(0)}=\varnothing$ otherwise. Let $\mathcal D^{(k)}\subseteq\mathcal F^{(k)}$ contain nodes blocked in Phase 1, with $\mathcal D^{(0)}=\varnothing$, and let $\mathcal U^{(k)}=\mathcal F^{(k)}\setminus(\mathcal E^{(k)}\cup\mathcal D^{(k)})$ contain nodes awaiting expansion or adoption.

The successful stopping test is
\[
 \textsc{Done}(\mathcal C_k)
 \quad\Longleftrightarrow\quad
 J_k\ne\varnothing,\quad \zeta_k\le\epsilon,\quad
 \overline C_k-\underline C_k\le\tau_C.
\]

The three conditions require a verified union, an omitted-mass bound within tolerance, and a sufficiently small complexity gap. Exhausting the permitted work is a separate stopping status, not a declaration of success.

The regional routines and \textsc{Update} return $\mathit{status}\in\{\textsc{ok},\textsc{incomplete},\textsc{failure}\}$. Here \textsc{ok} means completion, \textsc{incomplete} means a work limit prevented completion, and \textsc{failure} means an evaluation error, empty interval intersection, or reversed complexity bounds. Outputs are used only after an \textsc{ok} status. \textsc{CompleteLeaf} also returns a Boolean $\mathit{blocked}$, true exactly when the completed action defers the leaf to Phase 2. Such a deferral has status \textsc{ok} and preserves a valid interval.

\subsubsection{Shared updates and termination}
\label{sec:ars-shared-updates}

Both phases call \textsc{Update} after completing a regional evaluation or sampling batch. Each proposed update $k\ge1$ copies state $k-1$ and intersects new intervals with stored bounds. Unspecified quantities retain their previous values. \textsc{Update} first propagates mass bounds through every complete parent--child family. A bottom-up sweep intersects each parent interval with the sum of its child intervals. A top-down sweep intersects each child interval with the parent interval minus the sum of its siblings' intervals. Use at most four paired sweeps, stopping earlier if no endpoint changes. Then solve \eqref{eq:frontier-programs}, reusing solutions if their inputs are unchanged, and set
\[
 \underline C_k=\max\{\underline C_{k-1},\underline C^{(k)}\},\qquad
 \overline C_k=\min\{\overline C_{k-1},\overline C^{(k)}\}.
\]
If $J^{(k)}=\varnothing$, keep $(J_k,\zeta_k)=(J_{k-1},\zeta_{k-1})$. Otherwise set $(J_k,\zeta_k)$ to the pair with smaller complexity, then smaller omitted-mass bound, preferring the current selection on an exact tie. Selection uses frontier order and retains the previous dynamic-programming entry on equal benefit. On completion, the routine returns $(B^{(k)},S^{(k)},\mathcal C_k,\textsc{ok})$; otherwise it returns the status defined above. Commit only after the action and selection succeed. Discarded updates leave the saved state unchanged, but their work remains charged to the shared budgets. Propagation uses no additional samples. Appendix~\ref{app:implementation} gives the selection implementation and mass-conservation proof.

\noindent\textit{Termination and output.} The final branches return \textsc{success} when \textsc{Done} holds, \textsc{no-verified-union} when $J_k=\varnothing$, and \textsc{unresolved-gap} otherwise. Any unsuccessful initialization returns $(\mathcal C_{-1},\textsc{failure})$. Later failures return the last committed record with \textsc{failure}. Each returned record determines $A_{J_k}$ whenever $J_k\ne\varnothing$.

\subsubsection{Phase 1: Partition exploration}
\label{sec:ars-exploration}

Phase 1 expands internal nodes and evaluates or defers exposed leaves.

\emph{Choosing a node.} \textsc{BaseOpen} requires $\mathcal U^{(k)}\ne\varnothing$ and unexhausted score, action, and time allowances. It permits an attempt, whose completion is checked below. \textsc{Priority} selects a node maximizing $\log\overline Z_v^{(k)}-\log\underline c_v$ from $\mathcal U^{(k)}\cap S^{(k)}$ when nonempty, otherwise from $\mathcal U^{(k)}$. This favors large possible mass at low descendant complexity. Nodes with $\underline c_v=0$ take priority, and ties use current frontier order.

\emph{Initial regional bounds.} \textsc{AnchorBounds} evaluates one prescribed reference state (the anchor) $\gamma_{v,0}\in R_v$. The positive proposal $q_v$ and bound $\bar\rho_v\ge\osc_{\gamma\in R_v}\log\{s(\gamma)/q_v(\gamma)\}$ are constructed by rules fixed before the search. With $g_v(\gamma)=s(\gamma)/q_v(\gamma)$, it forms the deterministic interval
\[
 [e^{-\bar\rho_v}g_v(\gamma_{v,0}),\,e^{\bar\rho_v}g_v(\gamma_{v,0})]
\]
and, in BMA, tightens its upper endpoint with the direct Gaussian mass bound in Appendix~\ref{app:bma-initial-routines}. The interval is valid because $Z_v=\E_{q_v}g_v$ and the range bound holds throughout $R_v$. All children $\mathcal H(v)$ must be bounded before replacing an internal node $v$ in the frontier. A leaf whose endpoints coincide becomes eligible immediately. Appendix~\ref{app:bma-initial-routines} specifies the BMA proposal and range bound.

\looseness=-1
\emph{Completing a terminal leaf.} \textsc{CompleteLeaf} uses an exact-sum cap and an initial sample-count floor and ceiling, fixed before search. Appendix~\ref{app:bma-initial-routines} gives the BMA limits. The initial sample mean and relative-error radius are specified in Appendix~\ref{app:sampling}.

Choose $N_v$ as the smallest integer at least the sample floor satisfying $\eta_v(N_v)\le\eta_0$, with $N_v=+\infty$ if no finite count suffices. There are three outcomes.
\begin{itemize}[leftmargin=*,itemsep=2pt,topsep=3pt]
\item \emph{Exact summation.} If $\lvert R_v\rvert$ is within the exact-sum cap and either $\lvert R_v\rvert\le N_v$ or $N_v$ exceeds the sample ceiling, sum exactly and intersect with $[Z_v,Z_v]$.
\item \emph{Deferral to refinement.} Otherwise, if $N_v$ exceeds the sample ceiling, return the current interval with $\mathit{blocked}=\textsc{true}$.
\item \emph{Initial sampling.} In the remaining case, draw $N_v$ samples and intersect with the initial confidence interval in Appendix~\ref{app:sampling}.
\end{itemize}
Each completed branch returns \textsc{ok}, with $\mathit{blocked}=\textsc{false}$ for exact summation or initial sampling. Algorithm~\ref{alg:workflow} adds every completed leaf to $\mathcal E^{(k+1)}$ and each deferred leaf also to $\mathcal D^{(k+1)}$. An unfinished action returns \textsc{incomplete}, preserving the last committed state.

\subsubsection{Phase 2: Regional mass refinement}
\label{sec:ars-refinement}

A blocked leaf remains eligible because its existing interval is valid. Blocking means that the initial precision target requires too many draws under the Phase 1 rule. If the stopping test has not passed, Phase 2 targets these leaves using a new proposal and bounded-weight intervals. It neither expands the tree nor admits new leaves. With $k_0$ the last exploration update, keep $\mathcal F^{(k)}=\mathcal F^{(k_0)}$ and $\mathcal E^{(k)}=\mathcal E^{(k_0)}$, and refine $\mathcal R=\mathcal E^{(k_0)}\cap\mathcal D^{(k_0)}$.

\emph{Sampling checkpoints.} Prespecify cumulative sample sizes $0=N_0<N_1<\cdots<N_T$, with $N_1\ge2$ and $T\ge1$. For $v\in\mathcal R$, $t_v^{(k)}$ counts the checkpoints completed by that leaf at global update $k$, while $N_t$ counts draws at checkpoint $t$. The weight vector $\mathcal W_{v,t}=(W_{v,1},\ldots,W_{v,N_t})$ stores their weights, starting from $t_v^{(k_0)}=0$ and $\mathcal W_{v,0}=()$.

\emph{Choosing a batch.} Check the time allowance before each allocation. The total charge $b_v^{(k)}$ comprises $N_{t_v^{(k)}+1}-N_{t_v^{(k)}}$ new scores and, when $t_v^{(k)}=0$, one preparation score per computational cell. For BMA, the splitting rule in Appendix~\ref{app:bma-envelope-routine} determines the cell count before preparation. \textsc{Affordable} returns leaves $\mathcal Q^{(k)}\subseteq\mathcal R$ with positive interval width, $t_v^{(k)}<T$, and total charge $b_v^{(k)}$ within the remaining score budget. \textsc{Pick} maximizes
$a_v^{(k)}(\overline Z_v^{(k)}-\underline Z_v^{(k)})r_v^{(k)}/b_v^{(k)}$ over $\mathcal Q^{(k)}$.
Here $a_v^{(k)}=\epsilon/(1-\epsilon)$ for $v\in S^{(k)}$ and one otherwise, $r_v^{(k)}=1/2$ at the first visit and $1-\sqrt{N_{t_v^{(k)}}/N_{t_v^{(k)}+1}}$ on revisits, with ties resolved by the frontier order fixed at $k_0$. This allocation heuristic combines interval width, a nominal shrinkage factor, and batch work.

\emph{Preparing and sampling a leaf.} \textsc{PrepareEnvelope} constructs a proposal $q_v$ and bound $U_v$ such that $s(\gamma)\le U_vq_v(\gamma)$ on $R_v$. Thus $U_v$ bounds every importance weight and, by averaging, $Z_v$. Dividing weights by $U_v$ places them in $[0,1]$ for the checkpoint confidence intervals. The new $q_v$ replaces the Phase 1 proposal and stays fixed throughout the leaf's refinement stream. Appendix~\ref{app:bma-envelope-routine} gives the BMA construction.

\textsc{SampleWeights} draws $N_t-N_{t-1}$ fresh independent states from this $q_v$, where $t=t_v^{(k)}+1$, and returns a batch $w$ of weights $W_{v,i}=s(\gamma_{v,i})/[U_vq_v(\gamma_{v,i})]\in[0,1]$. Each has expectation $Z_v/U_v$. Append $w$ to $\mathcal W_{v,t-1}$ to obtain $\mathcal W_{v,t}$, with $\Vert$ denoting concatenation. Initial-run samples are not pooled into this stream. A nonfinite log weight or weight above one terminates with failure, without clipping.

\emph{Updating the search.} At each completed checkpoint, the regional confidence interval combines Hoeffding, empirical Bernstein, and binary Kullback--Leibler (KL) bounds on the mean of the normalized weights, then rescales to the regional mass. Appendix~\ref{app:sampling} gives the component formulas, failure allocation, and coverage argument. Intersect the resulting $\mathcal I_{v,t}$ with $B_v^{(k)}$ and $[0,U_v]$, advance only that leaf's counter, and call \textsc{Update} before refreshing affordability. Thus refinement uses all draws up to a declared checkpoint while preserving earlier bounds.

\Needspace{4\baselineskip}
\noindent\textit{Phase 2 illustration (hypothetical intervals).}
Suppose exploration ends at update $k_0$ with three eligible terminal regions $R_1,R_2,R_3$, unit complexities, $\epsilon=.10$, and $\tau_C=0$. Their unnormalized weight intervals are $[80,80]$, $[15,15]$, and $[1,20]$. Only the third leaf is blocked, and ancestor bounds impose no tighter leaf restrictions. All three regions give a verified union. One region cannot meet the target even under optimistic bounds, but two might, so the saved complexity bounds are two and three.

Assume a valid envelope $U_3=20$ and sufficient budget for two prespecified checkpoints. The regional intervals in Table~\ref{tab:phase2-trace} are stipulated for illustration, not calculated from simulated draws. They represent the checkpoint intervals after rescaling the bounded-weight mean by the envelope. Intersect each with the stored interval and $[0,U_3]$, then call \textsc{Update}.
\begin{table}[H]
\centering
\caption{Phase 2 illustration with stipulated intervals.}
\label{tab:phase2-trace}
\begin{tabular}{clll}
\toprule
Update & New interval & Stored interval for $Z_3$ & Saved record $\mathcal C_k$\\
\midrule
$k_0$ & --- & $[1,20]$ & $(\{1,2,3\},2,3,0)$\\
$k_0+1$ & $[3,12]$ & $[3,12]$ & $(\{1,2,3\},2,3,0)$\\
$k_0+2$ & $[2,7]$ & $[3,7]$ & $(\{1,2\},2,2,7/102)$\\
\bottomrule
\end{tabular}
\end{table}
At the first checkpoint, omitting the third region has bound $12/(95+12)>.10$, so no two-region union is verified and the saved record is unchanged. The second checkpoint uses the accumulated refinement stream. Intersecting its new interval with the previous one preserves the lower endpoint three and reduces the upper endpoint to seven. Now $7/(95+7)<.10$, and the verified union $R_1\cup R_2$ matches the lower complexity bound two. The stopping test passes. The saved omitted-mass bound rises from zero because a cheaper valid summary replaces the full union. No tree expansion or new leaf admission occurs.

\subsection{Gaussian BMA calculations}
\label{sec:ars-models}

Gaussian BMA supplies regional intervals through proposals, range bounds, and envelopes, summarized here and constructed in Appendix~\ref{app:bma-construction}. The network studies use supplied regional sums, with enumeration in Section~\ref{sec:network-simulation} and a separate numerical search variant for Sachs in Section~\ref{sec:sachs}.

\label{sec:ars-bma}

\emph{BMA inputs and model.} For data $y\in\mathbb R^n$ and $X\in\mathbb R^{n\times p}$, the likelihood is $y\mid\beta_\gamma,\gamma\sim N(X_\gamma\beta_\gamma,\sigma^2I_n)$. Take $\beta_\gamma\mid\gamma\sim N(0,\sigma^2\tau^2I_{|\gamma|})$ and independent predictor inclusion probability $\theta\in(0,1)$. Here $X_\gamma$ contains the selected columns, $\beta_\gamma$ their coefficients, $\sigma^2>0$ is the noise variance, and $\tau^2>0$ is the coefficient-prior scale. Integrating out $\beta_\gamma$ and multiplying by the support prior gives $s(\gamma)$. The supplied groups, terminal rule, and complexities define the fixed tree and objective (Section~\ref{sec:regions-cost}). Section~\ref{sec:adaptive-updates} specifies the controls and initializes the search state.

\emph{Initial mass bounds.} Small tables of supports within each predictor group define a proposal that samples groups independently. The highest-probability allowed group supports form the reference state. Bounds on variation of the importance weights turn one reference evaluation into the initial \textsc{AnchorBounds} interval. A direct Gaussian upper bound can tighten it further. Appendix~\ref{app:bma-initial-routines} gives the proposal, range calculations, numerical corrections, and \textsc{CompleteLeaf} limits, with derivations in Appendices~\ref{app:bma-interaction-derivation} and \ref{app:bma-path}.

\emph{Refinement envelopes.} \textsc{PrepareEnvelope} divides a terminal region into computational cells used only for sampling. A Gaussian inequality bounds every state weight in each cell. Tightening and combining these envelopes gives a proposal with bounded importance weights for Phase 2. The cells leave the terminal regions used for selection unchanged. The auxiliary optimization tightens the envelope while the regional selection objective remains fixed. Appendix~\ref{app:bma-envelope-routine} gives the cell construction, optimization settings, and numerical normalization, followed by their justification in Appendix~\ref{app:bma-envelope-validity}.

\Needspace{4\baselineskip}
\noindent\textit{Worked BMA example (exact arithmetic).} Consider the two-pair tree in Figure~\ref{fig:partition-tree}, with $G_1=\{1,2\}$ and $G_2=\{3,4\}$. For a hand-calculated illustration, take
\[
 X=I_4,\qquad y=0\in\mathbb R^4,\qquad
 \sigma^2=1,\quad \tau^2=3,\quad \theta=2/11.
\]
The marginal likelihood relative to the empty support is $2^{-|\gamma|}$ and the prior odds per included predictor are $2/9$. Hence $s(\gamma)\propto9^{-|\gamma|}$ and $\Pi(\gamma\mid D)=.1^{|\gamma|}.9^{4-|\gamma|}$. Use unit leaf complexities and descendant bounds, $K=13$, $\epsilon=.10$, $\tau_C=0$, and enough budget for root initialization and three expansions. Sampling controls may take any admissible values.

Each group's category probabilities are $(.81,.18,.01)$, so a terminal region's probability is the product of its two category probabilities. Orthogonality makes the group proposal equal the conditional posterior within every region. Thus $\bar\rho_v=0$ in exact arithmetic, \textsc{AnchorBounds} gives exact mass intervals, and exposed leaves become eligible immediately. The priority rule expands the root, then node $0$, then node $1$. Table~\ref{tab:bma-worked-trace} records the resulting search updates.
\begin{table}[H]
\centering
\caption{Search updates and saved records in the worked BMA example.}
\label{tab:bma-worked-trace}
\begin{tabular}{cll}
\toprule
$k$ & Completed action & $\mathcal C_k=(J_k,\underline C_k,\overline C_k,\zeta_k)$\\
\midrule
0 & Initialize at the root & $(\varnothing,1,+\infty,+\infty)$\\
1 & Expand the root & $(\varnothing,2,+\infty,+\infty)$\\
2 & Expand node $0$ & $(\varnothing,3,+\infty,+\infty)$\\
3 & Expand node $1$ & $(J_3,3,3,.0523)$\\
\bottomrule
\end{tabular}
\end{table}
At $k=2$, eligible leaves cover only branch $0$, with posterior mass $.81$, so no verified union is available. The largest two frontier masses sum to only $.8361$, while three suffice. At $k=3$, selection returns
\[
 J_3=\{(0,0),(0,1),(1,0)\},\qquad
 A_{J_3}=\{\varnothing,\{1\},\{2\},\{3\},\{4\}\}.
\]
These three regions contain five supports and retain mass $.6561+2(.1458)=.9477$. No two frontier nodes reach $.90$ (the largest pair has mass $.8019$), so the lower and upper complexity bounds both equal three. The omitted-mass bound is $\zeta_3=.0523\le\epsilon$, and the complexity gap is zero, giving success. Branch $2$ remains unexpanded and Phase 2 requires no draws.

\subsection{Guarantees and scope}
\label{sec:guarantees}
\label{sec:certificate-attainment}
\label{sec:computation}

\begin{samepage}
The guarantees require joint coverage of the stored regional intervals, conditional on the data and supplied tree:
\[
 \Pr\!\left\{Z_v\in B_v^{(k)}\ \text{for all }k\ge0,\ v\in\mathcal V^{(k)}\right\}\ge1-\delta.
\]
\end{samepage}
The tree may depend on the data but is fixed before fresh sampling, and the final state is held fixed after stopping. Appendix~\ref{app:sampling} establishes joint coverage for the stated failure allocation under adaptive sampling and stopping. Appendix~\ref{app:mass-cost-proof} proves the following result.

\begin{theorem}[Mass and representation-complexity guarantees]
\label{thm:search}
\begin{samepage}
Assume \resultcondition the tree is finite and the leaf complexities are nonnegative and additive; \resultcondition the deterministic regional-constant bounds and lower bounds on descendant complexity are valid; and \resultcondition\label{cond:search-coverage}  the stochastic intervals have simultaneous coverage at least $1-\delta$.
On the joint-coverage event above, fix any update $k$ with $\overline C^{(k)}<\infty$ and let $J^{(k)}$ solve the conservative problem in \eqref{eq:frontier-programs}. Then the following conclusions hold.\par
\end{samepage}
\begin{enumerate}[label=(\alph*),leftmargin=*,itemsep=2pt,topsep=3pt]
\item The omitted mass and total-variation error satisfy
\[
\TV(\Pi^{A_{J^{(k)}}},\Pi)
=1-\Pi(A_{J^{(k)}}\mid D)
\le
\zeta^{(k)}\le\epsilon.
\tag{5}\label{eq:search-guarantee}
\]
\item The representation complexity and its suboptimality satisfy
\[
  \underline C^{(k)}\le C^*_\epsilon\le C(A_{J^{(k)}})=\overline C^{(k)},\qquad
  0\le C(A_{J^{(k)}})-C^*_\epsilon\le \overline C^{(k)}-\underline C^{(k)}.
\]
If every candidate complexity is an integer multiple of some $h>0$ and the reported gap is smaller than $h$, then the incumbent is exactly optimal.
\end{enumerate}
\end{theorem}

On the same joint-coverage event, every update concerns the same fixed optimum, so the saved record satisfies $\TV(\Pi^{A_{J_k}},\Pi)\le\zeta_k\le\epsilon$ and $0\le C(A_{J_k})-C^*_\epsilon\le\overline C_k-\underline C_k$ whenever $J_k\ne\varnothing$, including at adaptive termination. Here $\epsilon$ controls posterior truncation, while $\delta$ concerns sampling randomness conditional on the data and tree. The TV guarantee describes the ideal conditional posterior $\Pi^{A_{J_k}}$. Subsequent approximate inference can introduce additional error beyond truncation.

\emph{Attaining the optimum.} Once every leaf is eligible, exact masses and $\underline c_v=c_v$ make both selection problems equal to Problem~\eqref{eq:problem}. Exact masses are not always necessary: Proposition~\ref{prop:certificate-attainment} in Appendix~\ref{app:method-strengthening} gives sufficient interval precision when an optimum lies strictly above the mass threshold. An optimum on the threshold may resist finite certification (Example~\ref{ex:zero-margin}), and a fixed budget need not expose all leaves or achieve the required precision.

\emph{Certificate scope and numerical implementation.} Theorem~\ref{thm:search} requires valid regional bounds. Conditional on valid deterministic bounds and the prescribed sampling protocol, $\delta$ controls sampling error in the stochastic intervals, excluding floating-point error in constructing their endpoints. For Gaussian BMA, local group tables provide the proposals and bounds derived in Appendices~\ref{app:bma-proposal}--\ref{app:bma-envelopes}. The implementation uses directed rounding in the GNU MPFR library to verify selection inequalities from supplied endpoints. Gaussian matrix calculations use floating-point padding and adjust bounds to the actual proposal probabilities. Because these calculations lack a rigorous arithmetic enclosure, the reported BMA certificates remain conditional on the computed endpoints being valid. Exhaustive audits check realized intervals, masses, and complexities on the analyzed datasets (Appendix~\ref{app:experiments}); they do not establish endpoint validity for all inputs.

\section{Simulation studies}
\label{sec:experiments}

The simulations assess retained mass, representation complexity, and certification. The BMA comparison gives ARS and MCMC the same regional family and asks how reliably each meets the mass target and approaches minimum complexity. The component study examines which operations improve these outcomes. The network experiment isolates the representation benefit of grouping graph orientations by supplying exact regional masses. Exact posterior calculations check each output. These studies assess regional selection; Tecator illustrates the resulting joint interpretation. Appendix~\ref{app:experiments} gives the full designs and outcomes.

\subsection{BMA design, comparison methods, and evaluation}
\label{sec:experiment-design}

The BMA comparison uses 200 independently generated regression datasets with $n=80$ observations each. Predictors occur in groups of two or four, with within-group population correlation $.95$. Strong correlation makes alternative predictors within a group difficult to distinguish. The 160 main datasets cross $p\in\{12,16,20,24\}$, the two group sizes, and between-group population correlations $.05$ and $.20$, with ten independent datasets in each of the 16 conditions. Three groups contain a signal. The 40 diffuse-signal stress datasets use 20 or 24 predictors in quartets and spread the signal over five groups, with ten datasets in each of four conditions. Increasing $p$ adds null groups within each design family. These designs do not require empirical orthogonality. With ten replications per condition, comparisons within individual conditions are descriptive.

Every method seeks a union of regions retaining at least $.90$ posterior mass. All methods use the same Gaussian posterior with known noise variance one, coefficient-prior variance $\tau^2=2$, and independent inclusion probability $.1$. They share the predictor groups supplied by the generating design, forward group order, and terminal regions. A terminal node has label $v=(a_1,\ldots,a_H)$, where $a_h=0,1,2$ specifies zero, one, or at least two predictors from group $h$. Its region can represent many models, and its complexity is $c_v=1+\sum_{h=1}^{H}(1+\mathbf 1\{a_h>0\})$. This prescribed rule assigns one unit per region, one per group, and one extra per active group, with equal weights for the one and $2+$ categories. For example, $v=(0,1,2,0)$ has complexity $1+4+2=7$. A union's complexity sums these regional weights within the supplied terminal partition.

The comparison uses our implementation of a four-chain model-space MCMC search, denoted MC$^3$, with single-variable toggles and active--inactive swaps. These are established variable-selection proposals \citep{clyde2011bas}. Each search produces two regional summaries. \emph{MC$^3$ frequencies} estimates a region's probability by the fraction of chain states in that region after burn-in. \emph{MC$^3$ visited scores} sums the posterior weights of distinct visited models in each region and divides by the total weight of all distinct visited models. Empirical frequencies and renormalized visited-model weights are established estimators \citep{clyde2012finite}. Here we combine each estimator with an independently implemented regional selector that seeks a minimum-complexity union with estimated mass at least $.90$. The two summaries share one MCMC search and differ only in their mass estimates. The comparison tests mass estimation, selection, and certification within the shared regional family.

Each dataset is analyzed at ceilings of 25,000 and 100,000 posterior-score requests, with a 300-second time allowance. Matched score budgets do not imply matched elapsed computing time. ARS uses failure allowance $\delta=.05$, complexity tolerance $\tau_C=.999999$, and at most 100,000 Phase 1 actions. Only after these searches finish do we enumerate all models to compute each output's actual retained mass and the minimum representation complexity. These exact references do not guide either method's search or stopping. Appendix~\ref{app:bma-scaling-design} gives the generation and MCMC protocols, recorded timings, execution and accounting details, and the location of the recorded seeds.

We distinguish three outcomes. \emph{Mass feasibility} means that the exact check confirms at least $.90$ retained mass. \emph{Attained optimality} means that a feasible output has the exact minimum complexity within the supplied family. \emph{Certified optimality} means that the method's own mass and complexity bounds establish feasibility and optimality conditional on valid computed regional bounds, with the numerical scope specified in Section~\ref{sec:computation}. An output may therefore attain the optimum without certifying it.

The \emph{complexity ratio} divides an output's complexity by the exact feasible minimum. For a feasible output, a ratio of one indicates minimum complexity, and larger values indicate a more complex description. We report this ratio only for outputs whose actual mass meets the target. The denominators for feasible and optimal counts include all attempted datasets, including searches that return no summary. These counts must accompany the conditional complexity summaries.

\subsection{Posterior mass, representation complexity, and scaling}
\label{sec:bma-scaling}

Table~\ref{tab:bma-scaling} reports the main-design results at 100,000 score requests and a 300-second time allowance. A dash means that the MCMC summary does not provide an optimality certificate, rather than that its output failed the exact check. Appendix~\ref{app:bma-scaling-design} reports the 25,000-request results and all mass shortfalls.

\begin{table}[!tb]
\centering
\caption{Main-design posterior mass, representation complexity, and certification at the 100,000-request ceiling.}
\label{tab:bma-scaling}
\small
\setlength{\tabcolsep}{5pt}
\begin{tabular*}{\textwidth}{@{\extracolsep{\fill}}lrrrr@{}}
\toprule
Method & Feasible & \shortstack{Optimum\\attained} & \shortstack{Optimum\\certified} & \shortstack{Median\\complexity ratio}\\
\midrule
\multicolumn{5}{l}{Main design, 160 datasets}\\
ARS & 148 & 102 & 96 & 1.00\\
MC$^3$ frequencies & 123 & 101 & --- & 1.00\\
MC$^3$ visited scores & 93 & 93 & --- & 1.00\\
\bottomrule
\end{tabular*}
\par\smallskip
\begin{minipage}{\textwidth}\small
\textit{Notes.} Feasibility and attained optimality are assessed by exhaustive post-search enumeration. Complexity-ratio medians use feasible outputs. Certification is conditional on valid computed regional bounds (Section~\ref{sec:computation}).
\end{minipage}
\end{table}

\begin{figure}[!tb]
\centering
\includegraphics[width=\textwidth]{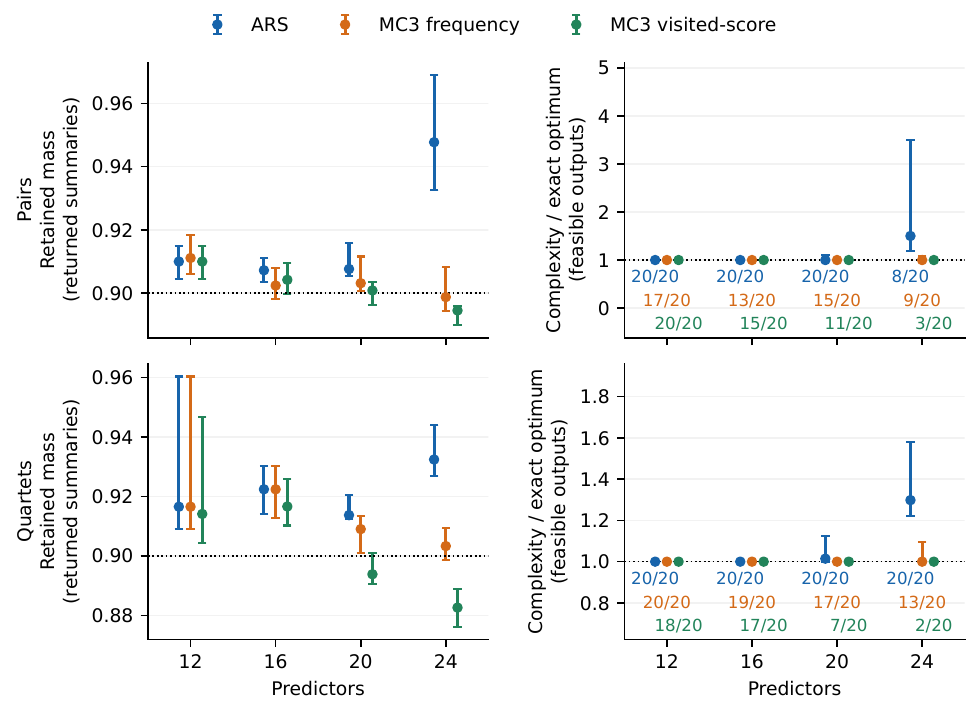}
\caption{Mass and representation complexity across regression dimensions at 100,000 score requests and a 300-second allowance. Points and bars give medians and interquartile ranges, with 20 main-design datasets per cell before conditioning, pooling the two between-group correlations. Mass summaries condition on returned outputs; complexity ratios condition on meeting the $.90$ mass target. Colored counts give feasible outputs out of all 20 attempts. Higher mass retains more probability but may require a larger summary; lower complexity ratios are better, with one optimal.}
\label{fig:bma-scaling}
\end{figure}

In the main design, ARS met the mass target more often than either MC$^3$ summary. Its advantage over MC$^3$ frequencies was 15.6 percentage points, with a 95\% stratified paired percentile bootstrap interval of 9.4--22.5 points. The two methods attained the optimum almost equally often, while ARS also supplied optimality certificates (Table~\ref{tab:bma-scaling}). The twelve ARS nonreturns all occurred at 24 predictors in pairs and reached the time allowance. Every other method--dataset combination returned a summary.

The diffuse-signal stress comparison favored ARS for mass coverage and frequency-based MC$^3$ for attained complexity, with no ARS optimum certified. Appendix~\ref{app:bma-scaling-design} and Table~\ref{tab:bma-mass-shortfalls}(a) give the full comparison.

The MC$^3$ mass shortfalls were generally small, so a missed target need not indicate a large probability error (Table~\ref{tab:bma-mass-shortfalls}(b)).

All saved ARS regional intervals, checkpoints, omitted-mass bounds, and complexity brackets passed the exact checks on these datasets.

\begin{samepage}
In Figure~\ref{fig:bma-scaling}, all three summaries often meet the target at minimum complexity at smaller dimensions. At 24 predictors, ARS complexities are more variable and the twelve paired-group searches return no summary. The MCMC summaries more often miss the mass target. A median complexity ratio near one therefore describes compactness among successful outputs and should be read together with the feasible count.
\end{samepage}

\FloatBarrier

\subsection{How grouping, refinement, and selection affect the result}
\label{sec:bma-confirmatory}

The component study examines how regional structure, interval precision, sampling allocation, and selection affect ARS. It uses 100 independent datasets in each of 24 settings, giving 2,400 datasets with pairs, triples, quartets, mixed group sizes, interacting random groups, and exact-orthogonality controls. Each dataset is analyzed at four retained-mass targets, from $.80$ to $.99$, giving 9,600 primary searches. Paired comparisons vary individual components on subsets of the same datasets. Repeated targets and variants are not additional independent datasets. The primary runs use BFGS envelope optimization, 25,000 score requests, and a 60-second allowance. Appendix~\ref{app:bma-confirmatory-details} specifies the designs and outcomes.

\begin{figure}[!tb]
  \centering
  \includegraphics[width=\textwidth]{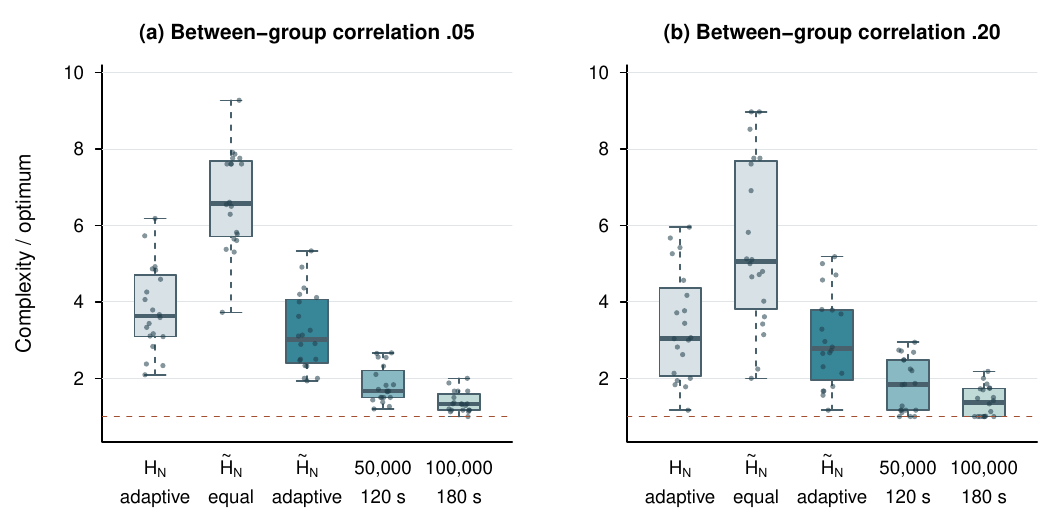}
\caption{Effects of refinement, allocation, and computational budgets on the forty paired quartet datasets at the $.90$ mass target, with twenty datasets per panel. Points give individual attained-to-optimal complexity ratios, boxes show quartiles, and the dashed line marks minimum complexity; every union is feasible. The first three conditions compare $H_N$ with $\widetilde H_N$ and equal with adaptive allocation at 25,000 requests and 60 seconds. The final two use $\widetilde H_N$ with adaptive allocation at 50,000 requests/120 seconds and 100,000 requests/180 seconds.}
  \label{fig:bma-confirmatory}
\end{figure}

The twenty-predictor quartet settings illustrate what a regional description represents. At the $.90$ mass target, the median output used 20 selected regions, compared with 692 posterior-ranked individual models. Each regional description collects alternative supports sharing a group-level pattern, so the two counts use different description units.

Paired runs on forty difficult quartet datasets measure the effects of interval precision and allocation. The KL-refined interval $\widetilde H_N$ reduced attained complexity by a mean 13.1 percent relative to the Hoeffding--Bernstein interval $H_N$. Allocating additional samples according to uncertainty in regional mass estimates reduced complexity by 43.6 percent relative to equal allocation. Increasing both score and time allowances reduced complexity further, although none of these higher-budget searches certified optimality. Figure~\ref{fig:bma-confirmatory} and Table~\ref{tab:bma-confirmatory-paired} report individual outcomes and uncertainty intervals for the paired comparisons. Every output in these comparisons met the mass target, so the complexity reductions measure improvements in compactness at adequate retained mass.

The orthogonal controls certified every optimum, and certification also occurred with interacting groups and with unequal or larger group sizes. The twenty-predictor comparisons distinguish two computational bottlenecks: exposing terminal regions in the paired designs and tightening their mass intervals in the quartet stress designs. Table~\ref{tab:bma-search-regimes} gives the full diagnostics.

Exploration order also affected finite-budget outcomes while preserving the terminal family and its optimum. Appendix~\ref{app:bma-confirmatory-details} reports the group-order and exploratory selector comparisons.

All available regional intervals, omitted-mass bounds, and complexity brackets in the component experiments pass their exact-reference checks, including stochastic refinement on 600 independent datasets. Appendix~\ref{app:bma-confirmatory-details} reports feasible and certified counts for all attempted searches.

\FloatBarrier

\subsection{Joint regional summaries in Bayesian networks}
\label{sec:network-simulation}

The network experiment asks how much a grouped description can reduce representation complexity when regional masses are known. Many directed acyclic graphs (DAGs) share the same undirected skeleton. A skeleton specifies which variables are connected and represents all compatible acyclic orientations. Here $\gamma$ is a DAG on the network vertices, while a search-tree node $v$ specifies edge-presence decisions. We use 200 independently simulated five-variable binary datasets, crossing sample sizes 100 and 300 with two signal strengths and 50 replications per condition. The posterior uses a Bayesian Dirichlet equivalent uniform score and a sparse graph prior \citep{heckerman1995}. Each skeleton is a terminal region with complexity one plus its number of edges.

\begin{figure}[!tb]
\centering
\includegraphics[width=\textwidth]{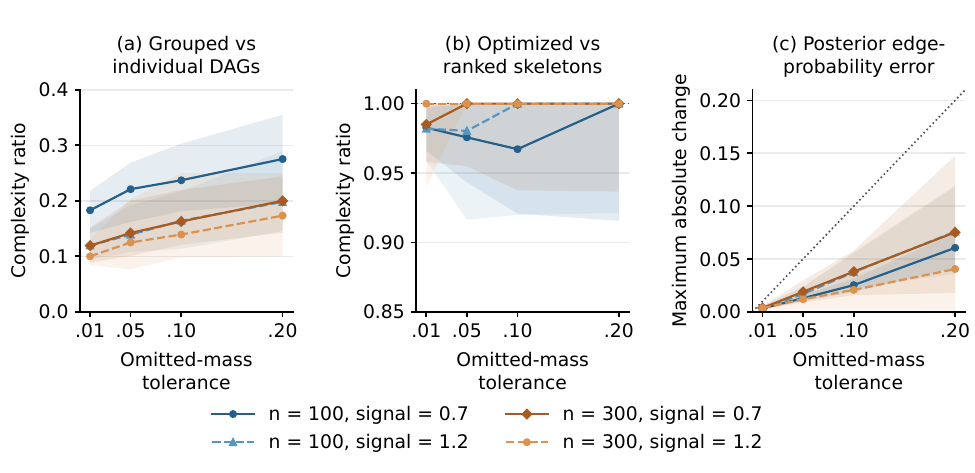}
\caption{Representation complexity and posterior error in Bayesian networks, using lexicographic tree order and the same 50 datasets per condition across all four tolerances. Panels (a) and (b) give median complexity ratios against posterior-ranked DAGs and skeletons. Panel (c) gives the mean maximum absolute change in posterior skeleton-edge probabilities. Shading gives 10th--90th percentiles, not confidence intervals. Dotted lines mark equal complexity in (b) and the omitted-mass tolerance in (c); vertical scales differ.}
\label{fig:network-compression}
\end{figure}

We enumerate all 29,281 DAGs and supply their regional sums to the tree selector. Four retained-mass targets, $.80$, $.90$, $.95$, and $.99$, and two fixed orders of the ten possible edges give 1,600 searches on the 200 datasets. We compare the selected union with lists formed by ranking individual DAGs or whole skeletons by posterior probability and retaining enough entries to meet the same target. Each listed DAG or skeleton has complexity one plus its number of edges. The DAG comparison measures the complexity reduction from grouping orientations. The skeleton comparison evaluates selection beyond simply ranking regions. We also compute the minimum-complexity solution over all 1,024 skeleton regions. Appendix~\ref{app:network-simulation} specifies the generative model, posterior, comparisons, and checks.

All 1,600 searches met their mass targets and closed their complexity bounds at the whole-family reference optimum. Both tree orders returned the same optimal complexity on every matched dataset and target, although ties can yield different retained unions. Figure~\ref{fig:network-compression} separates the grouping benefit from the additional effect of optimization. At the $.90$ target, median selected-to-ranked-DAG complexity ratios range from $.139$ to $.237$ across the four conditions: the grouped representation uses roughly 14--24 percent of the comparator's complexity under the declared representation-complexity measure. By comparison, mean complexity ratios against ranking whole skeletons range from $.964$ to $.999$. Grouping orientations supplies most of the reduction, while optimization removes the remaining avoidable complexity within that regional family.

The summaries also preserve posterior statements about graph adjacency. At the $.90$ target, condition-specific mean maximum changes in edge probabilities range from $.021$ to $.038$. Every recorded maximum change is bounded by the actual omitted mass, as predicted by the truncation result. This experiment demonstrates compact joint summaries and minimum-complexity selection in a non-regression state space. Because enumeration supplies the regional masses, it does not evaluate stochastic ARS mass estimation or scalability to larger networks.

\section{Regional uncertainty in real data}
\label{sec:real-data}

\subsection{Joint uncertainty among Tecator wavelength groups}

Tecator spectroscopy illustrates which joint wavelength-group patterns receive posterior support despite uncertainty about individual channels. ARS uses both stochastic phases, with mass-conservation propagation disabled. We use its union certificate to interpret the retained patterns and an exhaustive audit to quantify their probabilities and truncation effects. Appendix~\ref{app:tecator-ars} supplies sensitivities and computational details.

\label{sec:tecator-ars}

Tecator provides absorbance at 100 near-infrared wavelengths and measured fat content \citep{borggaard1992}. We use the documented split of 172 training and 43 test observations, with training-only standardization. Sixteen prespecified, equally spaced channels form Groups A--D, with original channel indices $(1,8,14,21)$, $(27,34,41,47)$, $(54,60,67,74)$, and $(80,87,93,100)$. In $X$, these columns are reindexed $1,\ldots,16$ in that order, so $G_h=\{4h-3,\ldots,4h\}$ and $\gamma\subseteq\{1,\ldots,16\}$. A retained region specifies zero, one, or at least two channels per group, expressing joint alternatives across wavelength ranges.

The Gaussian BMA posterior conditions on residual standard deviation $\sigma=.5$ on the standardized response scale, with $\tau^2=2$ and inclusion probability $.1$. Sensitivities use $\sigma=.25$ and $1$. Each setting defines a posterior conditional on the specified noise level. Each search has $\delta=.05$, 50,000 score calls, and a 180-second allowance. Enumeration of the 65,536 supports is reserved for auditing (Appendix~\ref{app:tecator-ars}).

At the 90 percent mass target, ARS returns eight regions with complexity 67 and a conservative omitted-mass bound of approximately $.0962$. Its lower complexity bound is 59. The search therefore establishes sufficient retained mass and shows that its complexity exceeds the feasible minimum by at most eight units, under the stated validity assumptions. These units follow the rule in Section~\ref{sec:experiment-design}: each region has complexity five plus the number of active groups. Five retained regions have three active groups and complexity eight each. The other three have four active groups and complexity nine each, for a total of $5\times8+3\times9=67$.

The search completes 79 stochastic checkpoints and uses 49,879 score calls. Post-search enumeration finds retained mass $.91013$ and optimal complexity 59. The audit confirms that an admissible union attains the lower complexity bound, while the returned eight-region summary is more complex.

Table~\ref{tab:tecator-patterns} states the eight joint alternatives at $\sigma=.5$ and the 90 percent target. Entries give channel counts within each group, with $2+$ denoting at least two. Post-search enumeration of the full reduced-channel posterior supplies the individual region masses and their ranking. Every retained support uses the first three groups, while the fourth may be absent, represented by one channel, or represented by several. The most probable pattern aggregates 484 supports with unequal weights. These probabilities describe joint support patterns under the specified model, not causal wavelength effects.

\begin{table}[!tb]
\centering
\caption{Retained Tecator joint patterns at the 90 percent mass target. Masses are from the exhaustive audit.}
\label{tab:tecator-patterns}
\small
\setlength{\tabcolsep}{9pt}
\begin{tabular*}{\textwidth}{@{\extracolsep{\fill}}ccccrr@{}}
\toprule
Group A & Group B & Group C & Group D & Supports & Mass\\
\midrule
$2+$ & $2+$ & 1 & 0 & 484 & .49137\\
$2+$ & $2+$ & 1 & 1 & 1,936 & .16185\\
$2+$ & $2+$ & $2+$ & 0 & 1,331 & .09699\\
$2+$ & $2+$ & $2+$ & 1 & 5,324 & .05228\\
$2+$ & 1 & 1 & 0 & 176 & .04769\\
$2+$ & $2+$ & 1 & $2+$ & 5,324 & .02914\\
1 & $2+$ & 1 & 0 & 176 & .02102\\
$2+$ & 1 & $2+$ & 0 & 484 & .00978\\
\bottomrule
\end{tabular*}
\end{table}

The union certificate itself supplies a joint posterior statement without enumeration. Let $E$ be the event that at least one channel is selected from each of Groups A, B, and C. Since every retained region lies in $E$, the certificate gives
$\Pi(E\mid D)\ge\Pi(A\mid D)\ge1-\zeta_k>.9037$,
where $k$ is the final update and $\zeta_k$ is the unrounded bound saved with the returned union $A=A_{J_k}$. Under the stated validity assumptions, the summary therefore certifies joint support for using all three groups without resolving which channels within them are selected.

To assess how this joint summary changes subsequent inference, we compare posterior inclusion probabilities and test predictions before and after truncation. The posterior probability that Group D is active changes from $.31994$ to $.26729$ after conditioning on the retained union. The largest channel-inclusion change is $.05206$, below both the actual omitted mass $.08987$ and its certified bound. Test root mean squared error is 4.271 percentage points for the truncated posterior mean and 4.275 for the full reduced-channel posterior mean. These post-search calculations use the enumerated reduced-channel posterior to assess truncation. The eight region descriptions represent 15,235 supports. Computing such summaries in larger problems requires regional probabilities and inference within each retained region, with any additional approximation error assessed separately from truncation.

Tightening the target to 95 percent increases the returned complexity to 101. The retained mass is $.95536$, the lower complexity bound is 92, and the audited optimum is 93. ARS certifies both target-specific optima at each of the other noise levels: complexities 43 and 44 at $\sigma=.25$, and 135 and 177 at $\sigma=1$. Appendix~\ref{app:tecator-ars} reports all six settings. The example illustrates how a usable joint summary and a valid mass bound can coexist with an unresolved complexity gap, and how the required representation changes with the posterior specification.

\subsection{Joint adjacency uncertainty in Sachs networks}
\label{sec:sachs}

The Sachs data provide 853 observational measurements of eleven cell-signaling variables, with a supplied three-level discretization \citep{sachs2005}. A state is a DAG, and a region contains all DAGs with the same undirected skeleton. Its probability measures support for a joint adjacency pattern while preserving uncertainty about edge directions. Each skeleton has representation complexity one plus its edge count. We use a BDeu likelihood with equivalent sample size one and prior probabilities $(.8,.1,.1)$ for absence and the two directions of each pair, conditioned on acyclicity.

\begin{figure}[!tb]
\centering
\includegraphics[width=\textwidth]{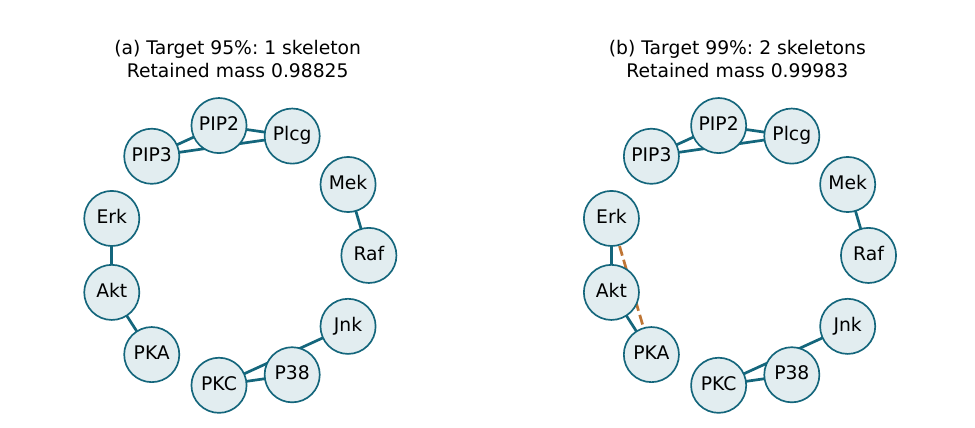}
\caption{Sachs joint adjacency summaries. In panel (b), solid edges occur in both retained skeletons; the dashed Erk--PKA edge occurs only in the second. Masses sum all compatible acyclic orientations.}
\label{fig:sachs-real}
\end{figure}

This application uses a separate numerical search variant. Dynamic programming sums regional weights and the full normalizer over all DAGs. Search resolves the 55 unordered pairs, expanding the largest-mass frontier region until completed skeletons cover the largest requested target or the budget is exhausted. Integer optimization selects among those skeletons, and a fractional relaxation and mandatory-leaf bound assess possible improvement over the full family. It uses neither stochastic mass refinement nor the priority and optimistic selection rules of Algorithm~\ref{alg:workflow}. Appendix~\ref{app:sachs-details} specifies the sums, bounds, and numerical checks.

At the 95 percent target, one eight-edge skeleton contains $.98825$ of the posterior. At 99 percent, two skeletons contain $.99983$, with complexities nine and nineteen, respectively (Figure~\ref{fig:sachs-real}). The second skeleton adds the Erk--PKA adjacency, with mass $.01158$. Thus the summary expresses joint support for a common adjacency pattern, with a smaller alternative containing one additional edge. Within each region, compatible acyclic orientations retain their relative posterior weights. These are statements under the specified observational graph posterior, not conclusions about causal directions.

A globally checked posterior-ranked list needs 104 individual DAGs to meet the 95 percent target, with complexity 936. Ranking whole skeletons attains the same complexities as optimization, so the benefit here is grouping orientation alternatives. Conditional on the numerical masses, the lower bounds establish minimum complexity and omitted mass controls total-variation error. Floating-point checks do not provide rigorous arithmetic enclosures, and no $1-\delta$ sampling guarantee is claimed for this variant. Prior sensitivities retain one to three skeletons (Appendix~\ref{app:sachs-details}).

\section{Discussion}
\label{sec:discussion}

Regional summaries connect interpretable joint patterns with quantitative statements about model uncertainty. A region represents alternative models together, preserving uncertainty within a shared pattern. ARS adds bounds on omitted mass and remaining complexity reduction. The mass bound controls total-variation error and supports statements shared by every retained region, as Tecator illustrates for the first three channel groups. A verified summary can therefore answer a joint scientific question even when its minimum complexity remains unresolved.

The grouping result explains when regional estimation can avoid enumerating individual models. Under the stated mass and sampling assumptions, its sufficient sample count is independent of region size at fixed accuracy and weight variation, even when an accurate model list must grow with that size. Conservative and optimistic selection connect this mass information to the quality of a retained union, while simultaneous intervals preserve validity under adaptive allocation and stopping.

The empirical findings distinguish mass feasibility, representation complexity, and certification. In BMA, ARS meets the mass target more often and additionally provides conditional certification, using the same regional family as MCMC. Frequency-based Monte Carlo more often attains minimum complexity in diffuse-signal cases. In the network simulation, grouping orientations supplies most representation reduction, with smaller gains from optimized selection. Sachs illustrates joint adjacency uncertainty, with ranking skeletons attaining the same complexity as optimization. These results show how grouping and selection contribute to compact summaries, while distinguishing attained quality from certification.

The framework's scientific usefulness depends on the supplied posterior, regions, and complexities. Meaningful groups make a short representation informative. The simulations use groups from the generating design, leaving learned or misspecified grouping untested, and Gaussian BMA conditions on a specified noise variance. The probability guarantees require valid regional bounds and the prescribed sampling protocol. The numerical implementation and its audits have the scope described in Section~\ref{sec:computation}, including the absence of rigorous arithmetic enclosure for the full calculation. Subsequent approximate inference needs a separate accuracy assessment.

The main BMA comparison uses 12--24 predictors, and stochastic Tecator uses sixteen selected channels. Reaching larger problems requires addressing both access to terminal regions and precision within them. The observed sensitivity to group order and interval refinement points to tighter bounds and more effective allocation. The decision-margin result specifies sufficient precision once all terminal regions are available, without guaranteeing attainment within a fixed budget. When a full posterior is readily available, regional probabilities can instead be aggregated directly. Comparisons with structured credible-set methods and sensitivity to group boundaries would clarify the scope of regional summaries. Rigorous numerical enclosures would strengthen the assurance attached to its summaries.

\noindent\textbf{Code availability.} Code, accompanying data, and saved experimental results are available in the \href{https://github.com/lihanqing1997/regional-summaries-of-bayesian-model-uncertainty}{GitHub repository}.

\begingroup
\interlinepenalty=10000
\bibliography{references}
\bibliographystyle{abbrvnat}
\endgroup

\appendix

\section{Proofs of the main results}
\label{app:proofs}

The proofs follow the three steps behind the guarantees. Appendices~\ref{app:truncation-proof} and \ref{app:grouping-proof} establish what a retained set preserves and when grouping can shorten its description. Appendix~\ref{app:sampling} shows that the regional intervals remain valid throughout adaptive search. On that simultaneous event, Appendix~\ref{app:mass-cost-proof} proves the mass and complexity guarantees, and Appendix~\ref{app:method-strengthening} gives sufficient precision for certifying optimality.

\subsection{Truncation, prediction, and posterior expectations}
\label{app:truncation-proof}

\begin{proof}[Derivation of the truncation identities]
Write $p_\gamma=\Pi(\gamma\mid D)$ and $\alpha=\Pi(A\mid D)>0$. Since $p^A_\gamma=p_\gamma\1\{\gamma\in A\}/\alpha$,
\[
 \TV(\Pi^A,\Pi)=\tfrac12\sum_{\gamma\in A}p_\gamma(\alpha^{-1}-1)
                  +\tfrac12\sum_{\gamma\notin A}p_\gamma=1-\alpha.
\]

For a common predictive kernel $K(\gamma,E)$, put $d_\gamma=p^A_\gamma-p_\gamma$. The positive and negative parts of $d$ each sum to $\TV(\Pi^A,\Pi)$. Since $K(\gamma,E)\in[0,1]$, $\lvert \sum_\gamma d_\gamma K(\gamma,E)\rvert \le\TV(\Pi^A,\Pi)$. Taking the supremum over predictive events proves contraction, also for continuous predictive outcomes.

For $\alpha<1$, the identity $\Pi=\alpha\Pi^A+(1-\alpha)\Pi^{A^c}$ gives
\[
\lvert \E_{\Pi^A}f-\E_\Pi f\rvert
 =(1-\alpha)\lvert \E_{\Pi^A}f-\E_{\Pi^{A^c}}f\rvert
 \le(1-\alpha)\osc_\G f.
\] The case $\alpha=1$ is immediate. Substituting $\alpha\ge1-\epsilon$ gives the bounds in Section~\ref{sec:target-sets}.
\end{proof}

\subsection{Grouping advantage}
\label{app:grouping-proof}

\begin{proof}[Proof of Proposition~\ref{prop:economy}]
Let $s_- =\min_{R_\star}s$ and $s_+=\max_{R_\star}s$, so $s_+/s_-\le e^\rho$. Each state in $R_\star$ has conditional probability at most $s_+/(Ms_-)\le e^\rho/M$. A subset $B$ therefore has $\Pi(B\mid D)\le e^\rho\lvert B\rvert/M+\xi$, allowing for mass at most $\xi$ outside $R_\star$. The truncation identity requires $\Pi(B\mid D)\ge1-\eta$, hence $\lvert B\rvert\ge e^{-\rho}(1-\eta-\xi)M$.

For uniform independent draws, $W_i=Ms(\gamma_i)$ has mean $Z(R_\star)$ and lies in $[Ms_-,Ms_+]$. If the endpoints coincide, every draw gives the exact constant. Otherwise Hoeffding's inequality yields
\[
 \Pr\bigl(\bigl\lvert \widehat Z_N/Z(R_\star)-1\bigr\rvert >u\bigr)
 \le2\exp\left\{-\frac{2Nu^2Z(R_\star)^2}{M^2(s_+-s_-)^2}\right\}
 \le2\exp\left\{-\frac{2Nu^2}{(e^\rho-1)^2}\right\},
\]
because $M(s_+-s_-)/Z(R_\star)\le s_+/s_--1$. This proves the stated sample bound and its independence of $M$ for fixed $\rho$. When $\rho=0$, the weights are constant.
\end{proof}

\subsection{Simultaneous validity of regional sampling}
\label{app:sampling}
\label{app:staged-sampling}

The two phases use a fixed bound $K\ge\lvert\mathcal V\rvert$ on all tree nodes, including unvisited ones. Initial evaluation receives failure allowance $\delta/2$, and refinement receives $\delta/2$. Designs may depend on earlier search information but are fixed before their fresh draws.

\noindent\textit{Initial bounds and sampling.}
\label{sec:partition-exploration}
\label{sec:regional-bounds}
\textsc{AnchorBounds} constructs the deterministic interval in Section~\ref{sec:ars-exploration} from the prescribed reference state and range bound. Its strictly positive importance-sampling proposal $q_v$ on $R_v$ is easy to sample and evaluate. Dividing by its sampling probabilities gives the corrected weight $g_v$, whose log variation over the region is $\rho_v$,
\[
g_v(\gamma)=s(\gamma)/q_v(\gamma),\qquad
\rho_v=\osc_{\gamma\in R_v}\log g_v(\gamma).
\tag{A.1}\label{eq:importance-range}
\]
Here $\osc_{R_v}$ is the maximum over $R_v$ minus the minimum, and the correction ensures $Z_v=\E_{q_v}g_v$. Let $\bar\rho_v\ge\rho_v$ be a known bound and $\gamma_{v,0}\in R_v$ a reference state at which $g_v$ is evaluated. Every log weight differs from $\log g_v(\gamma_{v,0})$ by at most $\bar\rho_v$. Averaging over $q_v$ gives
\[
e^{-\bar\rho_v}g_v(\gamma_{v,0})
\le Z_v
\le e^{\bar\rho_v}g_v(\gamma_{v,0}).
\tag{A.2}\label{eq:anchor-bound}
\]
Thus one reference evaluation gives a deterministic interval, provided the range bound holds throughout the region. An observed sample range is insufficient.

For $\delta_v=\delta/(2K)$ and $N$ fresh independent draws $\gamma_{v,i}\sim q_v$, define
\[
 \widehat Z_{v,N}=N^{-1}\sum_{i=1}^{N}g_v(\gamma_{v,i}),\qquad
 \eta_v(N)=(e^{\bar\rho_v}-1)\sqrt{\log(2/\delta_v)/(2N)}.
\]
For a completed initial run with $\eta_v(N_v)\le\eta_0<1$, intersect the deterministic interval with
$[\widehat Z_{v,N_v}/(1+\eta_v(N_v)),\,\widehat Z_{v,N_v}/(1-\eta_v(N_v))]$.
The proposal, sample size, range bound, and allowance are fixed before the fresh draws. Section~\ref{sec:ars-exploration} specifies the exact-sum, sampling, and deferral rules. The following lemma gives the general coverage statement, applied with total allowance $\delta/2$ in Phase 1.

\begin{lemma}[Simultaneous confidence bounds]
\label{lem:evidence}
Assume \resultcondition $\delta\in(0,1)$ and each node receives at most one final run; \resultcondition its positive proposal $q_v$, valid log-range bound $\bar\rho_v$, sample size $N_v\ge1$, and allowance $\delta_v\in(0,1)$ are fixed before its fresh conditionally IID final draws, possibly using past search information and a discarded pilot; and \resultcondition $\sum_v\delta_v\le\delta$ on every realized search path, where the sum is over nodes receiving a final run. Then, conditional on the data and tree, with probability at least $1-\delta$ every completed run satisfies
\[
 \lvert \widehat Z_{v,N_v}/Z_v-1\rvert \le \eta_v(N_v):=(e^{\bar\rho_v}-1)
                    \sqrt{\frac{\log(2/\delta_v)}{2N_v}}.
 \tag{A.3}\label{eq:simultaneous-radius}
\]
\end{lemma}

\begin{proof}
Index launched runs by $j$ and condition on the history $\mathscr H_j$ immediately before their fresh draws. For the node $v=v_j$ launched at run $j$, write $\delta_j=\delta_{v_j}$. The chosen design is then fixed, the weights have mean $Z_v$, and their minimum and maximum, $\underline g_v$ and $\overline g_v$, satisfy $(\overline g_v-\underline g_v)/Z_v\le\overline g_v/\underline g_v-1\le e^{\bar\rho_v}-1$. Conditional Hoeffding bounds the failure probability in \eqref{eq:simultaneous-radius} by $\delta_v$; constant weights give zero error.

There are at most $K$ runs for a fixed bound $K$ on the finite tree's size. Pad with inactive runs, whose failure events and allowances are zero. For an incomplete run, conceptually generate its remaining draws; a reported failure is contained in the corresponding full-run event $E_j$. Thus
\[
 \Pr\left(\bigcup_j E_j\right)
 \le\sum_j\E\{\Pr(E_j\mid\mathscr H_j)\}
 \le\E\sum_j\delta_j\le\delta.
\]
This argument does not condition on completion or require independence between designs. For $\eta_v(N_v)<1$, rearranging $(1-\eta_v(N_v))Z_v\le\widehat Z_{v,N_v}\le(1+\eta_v(N_v))Z_v$ gives the interval defined above; intersection with valid deterministic bounds preserves coverage.
\end{proof}

With the prescribed allocation $\delta_v=\delta/(2K)$, the target $\eta_v(N_v)\le\eta_0$ holds whenever
\[
 N_v\ge\max\left\{1,\left\lceil\frac{(e^{\bar\rho_v}-1)^2\log(4K/\delta)}{2\eta_0^2}\right\rceil\right\}.
\]
\textsc{CompleteLeaf} also enforces its declared sample floor and ceiling.

\noindent\textit{Refinement at declared checkpoints.}
\label{sec:regional-refinement}
Refinement uses a strictly positive proposal $q_v$ and finite envelope $U_v$ with $s(\gamma)\le U_vq_v(\gamma)$. Its normalized weights $W_{v,i}=s(\gamma_{v,i})/[U_vq_v(\gamma_{v,i})]\in[0,1]$ have mean $Z_v/U_v$. Appendix~\ref{app:bma-envelopes} constructs the Gaussian envelope.

Let $t=1,\ldots,T$ index the prespecified refinement checkpoints, with cumulative sample sizes $N_1<\cdots<N_T$. The fixed node set $\mathcal V$ satisfies $\lvert\mathcal V\rvert\le K$. Allocate $\alpha=\delta/(2KT)$ to each node and checkpoint, where $\delta/2$ is the reserved refinement allowance.

The interval $H_N$ intersects Hoeffding and empirical Bernstein bounds. Its refinement $\widetilde H_N$ adds a binary-KL interval, splitting the allowance equally between the two components.

For the chosen leaf, write $N=N_t$ and $W_i=W_{v,i}$. Let $m_N=N^{-1}\sum_iW_i$ and $\widehat V_N=(N-1)^{-1}\sum_i(W_i-m_N)^2$ be the sample mean and variance. At allowance $\beta\in(0,1)$, define the Hoeffding and empirical Bernstein radii $r_{1,N}$ and $r_{2,N}$, then intersect their intervals with a binary-KL interval:
\[
\begin{aligned}
 r_{1,N}(\beta)&=\sqrt{\frac{\log(4/\beta)}{2N}},\quad
 r_{2,N}(\beta)=\sqrt{\frac{2\widehat V_N\log(8/\beta)}{N}}
                    +\frac{7\log(8/\beta)}{3(N-1)},\\
 H_N(\beta)&=[m_N-\min\{r_{1,N}(\beta),r_{2,N}(\beta)\},
              m_N+\min\{r_{1,N}(\beta),r_{2,N}(\beta)\}]\cap[0,1],\\
 Q_N(\beta)&=\{q\in[0,1]:Nd(m_N,q)\le\log(2/\beta)\},\quad
 \widetilde H_N(\alpha)=H_N(\alpha/2)\cap Q_N(\alpha/2).
\end{aligned}
\]
Here $d(a,b)=a\log(a/b)+(1-a)\log\{(1-a)/(1-b)\}$ is binary relative entropy, with its limiting endpoint values.

Fix each proposal and envelope before its first fresh independent and identically distributed (IID) draw, and inspect the resulting stream only at the prescribed sizes. At a completed checkpoint, define $\mathcal I_{v,t}=U_v\widetilde H_{N_t}(\alpha;\mathcal W_{v,t})$, evaluated on the accumulated normalized weights $\mathcal W_{v,t}=(W_{v,1},\ldots,W_{v,N_t})$. Initial-run samples are excluded. For each prescribed prefix, conditioning on the history before the stream gives failure probability at most $\alpha$. Set $\mathcal I_{v,t}=[0,\infty)$ for an uncompleted checkpoint. Then
\[
  \Pr\left(\bigcup_{v\in\mathcal V}\bigcup_{t=1}^{T}
       \{Z_v\notin\mathcal I_{v,t}\}\right)
  \le KT\alpha=\delta/2.
  \tag{A.4}\label{eq:checkpoint-coverage}
\]
The proof below covers all prescribed prefixes under adaptive interleaving and stopping. Intersecting with earlier valid intervals and combining the two phases preserves total failure allowance $\delta$.

At a fixed size $N\ge2$, Hoeffding gives failure at most $\beta/2$ for radius $r_{1,N}(\beta)$. Applying the empirical Bernstein bound of \citet[Theorem~4]{maurer2009} to both $W_i$ and $1-W_i$, with allowance $\beta/4$ for each tail, gives radius $r_{2,N}(\beta)$ and combined failure at most $\beta/2$. Its sample variance equals ours because $\sum_{i<j}(W_i-W_j)^2=N\sum_i(W_i-m_N)^2$. The intersection $H_N(\beta)$ therefore has failure at most $\beta$.

For the binary-KL interval, convexity gives $\E e^{\lambda W}\le1-\mu+\mu e^\lambda$ for every real $\lambda$ when $W\in[0,1]$ has mean $\mu$ \citep[Lemma~9]{garivier2011}. Independence and the Chernoff argument give $\Pr(m_N\ge x)\le e^{-Nd(x,\mu)}$ for $x>\mu$, and the corresponding lower-tail bound for $x<\mu$. For $0<\mu<1$, optimizing uses $\lambda=\log\{x(1-\mu)/(\mu(1-x))\}$. The function $d(x,\mu)$ decreases to zero below $\mu$ and increases above it, so the two tails of $Nd(m_N,\mu)>\log(2/\beta)$ have total probability at most $\beta$. Endpoints follow by limits; if $\mu=0$ or $1$, observations equal that endpoint almost surely. Thus $Q_N(\beta)$ has coverage at least $1-\beta$, and $\widetilde H_N(\alpha)=H_N(\alpha/2)\cap Q_N(\alpha/2)$ has coverage at least $1-\alpha$. The intervals need not be independent.

To obtain simultaneous coverage, index regions by their order $\ell$ of starting refinement, padding to $K$ inactive positions. Condition on the history before region $\ell$'s first fresh draw, which fixes its proposal and envelope. Conceptually extend its IID stream to $N_T$ even if search stops sooner. Each prescribed prefix $N_t$ has conditional failure at most $\alpha$ by the preceding fixed-size calculation. Iterated expectation and a union bound over all $KT$ prefixes give failure at most $KT\alpha=\delta/2$. Every inspected checkpoint is one of these prefixes, regardless of adaptive interleaving or stopping. Adding the base allowance yields total failure at most $\delta/2+\delta/2=\delta$. Intersections with earlier valid intervals, including deterministic bounds, preserve this same event.

\subsection{Mass and representation-complexity guarantees}
\label{app:mass-cost-proof}

\begin{proof}[Proof of Theorem~\ref{thm:search}]
Fix any update $k$ on the simultaneous event $\mathcal G$ where every interval used by the search is valid, and use its current frontier and bounds. The preceding subsection supplies this event with probability at least $1-\delta$ under the stated sampling protocols. Deterministic intersections preserve validity, and replacing a node by its complete child set preserves the frontier partition.

For any truly feasible union $A\in\mathcal A$, let $\mathcal S^{(k)}(A)$ be the frontier nodes containing selected descendant leaves. Writing $R=Z(A)>0$ and $O=Z(A^c)$, disjointness and validity give $\underline Z^{(k)}(\mathcal F^{(k)}\setminus \mathcal S^{(k)}(A))\le O$ and $\overline Z^{(k)}(\mathcal S^{(k)}(A))\ge R$. Since $o/(r+o)$ increases with $o$ and decreases with $r$,
\[
 \frac{\underline Z^{(k)}(\mathcal F^{(k)}\setminus \mathcal S^{(k)}(A))}
      {\overline Z^{(k)}(\mathcal S^{(k)}(A))+\underline Z^{(k)}(\mathcal F^{(k)}\setminus \mathcal S^{(k)}(A))}
 \le O/(R+O)\le\epsilon.
\]
Thus $\mathcal S^{(k)}(A)$ passes the optimistic test. The selected descendants of its distinct nodes are disjoint, so their complexity bounds sum to at most $C(A)$. Minimizing first over optimistic candidates and then over truly feasible $A$ gives $\underline C^{(k)}\le C^*_\epsilon$.

Conversely, a nonempty conservatively feasible $J\subseteq\mathcal E^{(k)}$ selects whole terminal regions. Its actual retained and omitted constants $R_J=Z(A_J)$ and $O_J=Z(A_J^c)$ satisfy $R_J\ge\underline Z^{(k)}(J)$ and $O_J\le\overline Z^{(k)}(\mathcal F^{(k)}\setminus J)$. The same monotonicity gives
\[
 \Pi(A_J^c\mid D)\le
 \frac{\overline Z^{(k)}(\mathcal F^{(k)}\setminus J)}
      {\underline Z^{(k)}(J)+\overline Z^{(k)}(\mathcal F^{(k)}\setminus J)}\le\epsilon.
\]
The denominator is positive whenever the complement is nonempty, by positivity of state weights. If the entire frontier is retained with zero lower mass, the stipulated ratio and actual omitted mass are both zero. Hence a current conservative optimizer is truly feasible and has $C^*_\epsilon\le C(A_{J^{(k)}})=\overline C^{(k)}$. Equation~\eqref{eq:truncation-errors} gives the TV guarantee, and subtraction gives the complexity-gap bound. If all candidate complexities are multiples of $h>0$, any positive suboptimality is at least $h$, so a smaller reported gap certifies optimality.

These arguments hold at every frontier on the same event $\mathcal G$. The fixed optimum gives $\underline C_k\le C^*_\epsilon$ by the running maximum. When $J_k\ne\varnothing$, its saved mass bound remains valid and $\overline C_k=C(A_{J_k})\ge C^*_\epsilon$. Thus $\mathcal C_k$ satisfies the stated guarantees at every update and at adaptive termination. No further probability allowance is needed for saving or stopping.
\end{proof}

\subsection{Decision margins and sufficient precision}
\label{app:method-strengthening}
\label{app:attainment-proof}

\begin{proposition}[Attaining minimum representation complexity]
\label{prop:certificate-attainment}
At an update $k$, use the current intervals $[\underline Z_v^{(k)},\overline Z_v^{(k)}]$ for the terminal constants $Z_v$. For a nonempty leaf selection $J\subseteq\mathcal L$, define $H(J)=\epsilon Z(A_J)-(1-\epsilon)Z(A_J^c)$ and $\omega^{(k)}(J)=\epsilon\sum_{v\in J}(\overline Z_v^{(k)}-\underline Z_v^{(k)})+(1-\epsilon)\sum_{v\notin J}(\overline Z_v^{(k)}-\underline Z_v^{(k)})$.
Assume \resultcondition $0<\epsilon<1$, every terminal region is exposed and eligible, and $\underline c_v=c_v$ for every leaf; \resultcondition the intervals $[\underline Z_v^{(k)},\overline Z_v^{(k)}]$ are finite and simultaneously valid; \resultcondition\label{cond:attainment-optimum}  some optimal leaf selection $J^*$ satisfies $H(J^*)>\omega^{(k)}(J^*)$; and \resultcondition\label{cond:attainment-cheaper}  every nonempty leaf selection $J$ of strictly lower complexity satisfies $-H(J)>\omega^{(k)}(J)$. Then
$\underline C^{(k)}=\overline C^{(k)}=C^*_\epsilon$.
In particular, put $m_+=H(J^*)>0$ and
$m_-=\min_{C(A_J)<C^*_\epsilon}\{-H(J)\}$, with $m_-=+\infty$ if there is no nonempty union of lower complexity. The condition
$\sum_v(\overline Z_v^{(k)}-\underline Z_v^{(k)})<\min\{m_+,m_-\}$ is sufficient.
\end{proposition}

$H(J)$ is the signed feasibility margin, and $\omega^{(k)}(J)$ bounds its endpoint uncertainty. The finite family makes $m_->0$ whenever a cheaper union exists, but an optimum may have zero margin (Example~\ref{ex:zero-margin}). The result requires all terminal regions to be available and sufficient precision, neither of which is assured by a fixed budget.

\begin{proof}[Proof of Proposition~\ref{prop:certificate-attainment}]
At update $k$, on a valid complete terminal partition, define the endpoint margins
\[
  H_-^{(k)}(J)=\epsilon\sum_{v\in J}\underline Z_v^{(k)}-(1-\epsilon)\sum_{v\notin J}\overline Z_v^{(k)},\qquad
  H_+^{(k)}(J)=\epsilon\sum_{v\in J}\overline Z_v^{(k)}-(1-\epsilon)\sum_{v\notin J}\underline Z_v^{(k)}.
\] The conservative and optimistic tests are respectively $H_-^{(k)}(J)\ge0$ and $H_+^{(k)}(J)\ge0$. Endpoint validity gives $0\le H(J)-H_-^{(k)}(J)\le\omega^{(k)}(J)$ and $0\le H_+^{(k)}(J)-H(J)\le\omega^{(k)}(J)$.

Condition~(\ref{cond:attainment-optimum}) implies $H_-^{(k)}(J^*)>0$, admitting an optimum in the conservative problem. Condition~(\ref{cond:attainment-cheaper}) implies $H_+^{(k)}(J)<0$ for every union of lower complexity, excluding each from the optimistic problem. Since both problems now use actual leaf complexities, $\overline C^{(k)}\le C^*_\epsilon\le\underline C^{(k)}$. Theorem~\ref{thm:search} gives the reverse inequalities. Finally, $\omega^{(k)}(J)\le\sum_v(\overline Z_v^{(k)}-\underline Z_v^{(k)})$ for every union. The finite collection of unions with lower complexity has strictly negative margins, so its minimum negative-margin magnitude is positive when it is nonempty. This proves the total-width condition, with the stated $+\infty$ convention otherwise.
\end{proof}

\begin{example}{Zero margin}
\label{ex:zero-margin}
A zero margin can prevent finite certification. Take two leaves with constants $(.9,.1)$, complexities $(1,2)$, and $\epsilon=.1$. The unique optimum $\{1\}$ has complexity one and zero margin. Valid intervals $[.9-r,.9+r]$ and $[.1-r,.1+r]$, for any $0<r<.1$, give conservative margin $-r$ and optimistic margin $r$ for that union. Only the full union passes the conservative test, leaving complexity bounds one and three even as $r\downarrow0$. Additional exact information can remove this obstruction; shrinking positive widths alone need not.
\end{example}

For a sample-size consequence under Algorithm~\ref{alg:workflow}, assume $0<\epsilon<1$ and finite, simultaneously valid intervals, with $\underline c_v=c_v$ for every leaf as in Proposition~\ref{prop:certificate-attainment}. Suppose every terminal leaf is eligible at the Phase 2 start $k_0$, and let $\mathcal R\ne\varnothing$ be its fixed refinement set. Define the residual width outside that set by
$\Omega^{(k_0)}=\sum_{v\in\mathcal L\setminus\mathcal R}(\overline Z_v^{(k_0)}-\underline Z_v^{(k_0)})$.
Stored intersections cannot increase this contribution. If every $v\in\mathcal R$ has completed a declared checkpoint $N=N_t\ge2$ by update $k$, the Hoeffding radius $\sqrt{\log(8/\alpha)/(2N)}$ gives
\[
 \sum_{v\in\mathcal L}(\overline Z_v^{(k)}-\underline Z_v^{(k)})
 \le\Omega^{(k_0)}+\sqrt{\frac{2\log(16KT/\delta)}{N}}\sum_{v\in\mathcal R}U_v.
 \tag{A.5}\label{eq:attainment-hoeffding-width}
\]
For $m=\min\{m_+,m_-\}>\Omega^{(k_0)}$, a sufficient completed checkpoint is therefore
\[
 N>\frac{2(\sum_{v\in\mathcal R}U_v)^2}{(m-\Omega^{(k_0)})^2}\log(16KT/\delta).
 \tag{A.6}\label{eq:attainment-sample-size}
\]
For unequal checkpoints $N_{t_v^{(k)}}\ge2$, the sampling term becomes $\{2\log(16KT/\delta)\}^{1/2}\sum_{v\in\mathcal R}U_vN_{t_v^{(k)}}^{-1/2}$. Full refinement, or exact intervals outside $\mathcal R$, gives $\Omega^{(k_0)}=0$. If $m\le\Omega^{(k_0)}$, this bound alone cannot ensure certification by sampling $\mathcal R$, although propagation may reduce the residual width. With $\mathcal R=\varnothing$, only the stored-width condition in Proposition~\ref{prop:certificate-attainment} applies. These conclusions use the same simultaneous event and require checkpoints already covered by the failure allocation. Neither the grid nor the budget is guaranteed to reach them.

\section{Supporting mathematical results and constructions}
\label{app:construction}
\label{app:bma-construction}

This appendix provides the application-specific calculations used by ARS. Appendices~\ref{app:bma-proposal} and \ref{app:bma-path} construct the importance-weight range bound used in Phase 1. Appendix~\ref{app:bma-envelopes} constructs and optimizes the upper envelope used in Phase 2. Appendix~\ref{app:implementation} gives selection, sampling allocation, and reuse of parent--child bounds. All analytical inequalities concern exact quantities. Section~\ref{sec:computation} states the numerical scope of the implementation.

\subsection{Gaussian proposals and interaction bounds}
\label{app:bma-proposal}
\label{app:bma-spectral}
\label{app:bma-blocks}

For Gaussian BMA, both phases construct regional bounds from tables of supports within each predictor group. The tables serve different purposes in the two phases: Phase 1 bounds the variation of importance weights, while Phase 2 bounds their maximum after normalization. We first state the common posterior model. Here the state $\gamma$ is a predictor support, $\gamma\subseteq\{1,\ldots,p\}$. Let $y\in\mathbb R^n$ be the response, $X_{\gamma}$ the selected design columns, and $\beta_{\gamma}$ their coefficients. Take $y\mid\beta_{\gamma},{\gamma}\sim N(X_{\gamma}\beta_{\gamma},\sigma^2I_n)$ with known noise variance $\sigma^2$, coefficient prior $\beta_{\gamma}\mid {\gamma}\sim N(0,\sigma^2\tau^2I)$ with scale $\tau^2>0$, and independent Bernoulli support priors. Here $I_n$ is the $n\times n$ identity matrix, and $I$ has the dimension determined by context.

\subsubsection{Initial proposal and mass bounds}
\label{app:bma-initial-routines}

The following recipe gives the fixed \textsc{AnchorBounds} settings for the main BMA comparison. For \textsc{CompleteLeaf}, the exact-sum cap is 256 states, and the initial sample-count floor and ceiling are 100 and 2500 draws.

\looseness=-1
\emph{Local proposal.} Use the Gaussian model above with design $X\in\mathbb R^{n\times p}$ and common inclusion probability $\theta\in(0,1)$. For a node $v=(a_1,\ldots,a_d)$, its allowed local supports are $\mathcal U_{vh}=\{\gamma'\subseteq G_h:\min\{2,|\gamma'|\}=a_h\}$ when $h\le d$, and $\mathcal U_{vh}=2^{G_h}$ otherwise. Thus $\gamma\in R_v$ exactly when $\gamma\cap G_h\in\mathcal U_{vh}$ for every group. For Figure~\ref{fig:partition-tree} at $v=(1)$, $\mathcal U_{v1}=\{\{1\},\{2\}\}$ and $\mathcal U_{v2}=2^{\{3,4\}}$. For a predictor support $\gamma'$, $X_{\gamma'}$ contains its selected design columns. Set $C_{\gamma'}=I_n+\tau^2X_{\gamma'}X_{\gamma'}^\top$, $c_\sigma=(2\pi\sigma^2)^{-n/2}$, and $\ell_0(\gamma')=-(1/2)\log\det C_{\gamma'}+(\lVert y\rVert_2^2-y^\top C_{\gamma'}^{-1}y)/(2\sigma^2)$. Here $\sigma^2C_{\gamma'}$ is the $n\times n$ marginal covariance after integrating coefficients, and $\ell_0$ is the log likelihood relative to the empty support. For the empty support, $C_\varnothing=I_n$ and $\ell_0(\varnothing)=0$. For $\gamma'\subseteq G_h$, define $\pi_h(\gamma')=\theta^{|\gamma'|}(1-\theta)^{|G_h|-|\gamma'|}$, $w_h(\gamma')=\pi_h(\gamma')e^{\ell_0(\gamma')}$, and $q_{vh}(\gamma')=w_h(\gamma')/\sum_{U\in\mathcal U_{vh}}w_h(U)$ on $\mathcal U_{vh}$, with zero probability outside. Draw group supports independently from $q_v(\gamma)=\prod_hq_{vh}(\gamma\cap G_h)$. The reference state takes each group's highest-probability allowed support, with fixed table order for ties. Thus $s(\gamma)=c_\sigma e^{-\lVert y\rVert_2^2/(2\sigma^2)}e^{\ell_0(\gamma)}\prod_h\pi_h(\gamma\cap G_h)$, and only $\ell_0(\gamma)-\sum_h\ell_0(\gamma\cap G_h)$ varies in the ideal log importance weight.

\looseness=-1
\emph{Range calculation.} Set $\bar\rho_v=0$ for a singleton region. Otherwise fix $v$, write $\mathcal U_h=\mathcal U_{vh}$, and let $\mathcal H=\{h:|\mathcal U_h|>1\}$. For $h\notin\mathcal H$, write $\mathcal U_h=\{\gamma_h^\circ\}$ and set $F=\bigcup_{h\notin\mathcal H}\gamma_h^\circ$. Thus $F$ contains exactly the predictors fixed as selected. Set $W_F=C_F^{-1}$, $\kappa=\tau^2/(2\sigma^2)$, $A_{\gamma'}=I_{|\gamma'|}+\tau^2X_{\gamma'}^\top W_FX_{\gamma'}$, $b_{\gamma'}=X_{\gamma'}^\top W_Fy$, and $\ell_F(\gamma')=-(1/2)\log\det A_{\gamma'}+\kappa b_{\gamma'}^\top A_{\gamma'}^{-1}b_{\gamma'}$. The matrix $W_F$ adjusts for fixed predictors. Adjusted inner products form the $|\gamma'|\times|\gamma'|$ matrix $A_{\gamma'}$ and length-$|\gamma'|$ vector $b_{\gamma'}$. The empty-support conventions give $\ell_0(\varnothing)=\ell_F(\varnothing)=0$. For $U=\bigcup_{h\in\mathcal H}G_h$, form $\mathsf D=\operatorname{blockdiag}_{h\in\mathcal H}A_{G_h}$. Let $\lambda_-$ be the smaller of one and the smallest eigenvalue of $\mathsf D^{-1/2}A_U\mathsf D^{-1/2}$, and $\lambda_+$ the larger of one and its largest eigenvalue. Put $\Delta=\sum_{h\in\mathcal H}\osc_{\gamma'\in\mathcal U_h}\{\ell_F(\gamma')-\ell_0(\gamma')\}$, $\bar d=\sum_{h\in\mathcal H}\max_{\gamma'\in\mathcal U_h}|\gamma'|$, and $\bar Q=\sum_{h\in\mathcal H}\max_{\gamma'\in\mathcal U_h}b_{\gamma'}^\top A_{\gamma'}^{-1}b_{\gamma'}$. These scalars bound the correction's oscillation, selected dimension, and summed groupwise quadratic terms, respectively. The spectral bound is $\rho_1=\Delta-(\bar d/2)\log\lambda_-+\kappa\bar Q(\lambda_-^{-1}-\lambda_+^{-1})$.

\looseness=-1
For tighter bounds, let $\mathscr P=(\mathcal P_1,\ldots,\mathcal P_m)$ partition the unresolved group indices $\mathcal H$ into blocks used for bounding interactions. A block support is $\gamma_i=\bigcup_{h\in\mathcal P_i}\gamma'_h$ with each $\gamma'_h\in\mathcal U_h$. All block maxima and oscillations below range over these choices. Define $\Delta_{\mathscr P}=\sum_i\osc_{\gamma_i}\{\ell_F(\gamma_i)-\sum_{h\in\mathcal P_i}\ell_0(\gamma'_h)\}$, $\chi_i=\max_{\gamma_i}\sqrt{b_{\gamma_i}^\top A_{\gamma_i}^{-1}b_{\gamma_i}}$, and $\chi=(\chi_1,\ldots,\chi_m)^\top\in\mathbb R^m$. Form the $m\times m$ interaction-bound matrix $\mathsf R$ by \eqref{eq:bma-coherence}, using the operator norm. With $\lambda_R=\lambda_{\max}(\mathsf R)<1$, the block bound is $\rho_2(\mathscr P)=\Delta_{\mathscr P}-(\bar d/2)\log(1-\lambda_R)+2\kappa \chi^\top \mathsf R(I-\mathsf R)^{-1}\chi$. Recompute these quantities for each partition. Form blocks greedily in forward and reverse group order, each with at most 16 joint configurations. When every unresolved group allows exactly one predictor, also form blocks of up to four consecutive groups. In that case, recomputing $\mathsf R$, $\chi$, and $\lambda_R$ with one group per block gives $\rho_3=\Delta-(1/2)\log\det(I-\mathsf R)+2\kappa \chi^\top \mathsf R(I-\mathsf R)^{-1}\chi$ if $\lambda_R<1$. Set a bound to $+\infty$ when its conditions fail.

A path bound supplies a further candidate. Let $K_x=\max_j\lVert x_j\rVert_2$, where $x_j$ is design column $j$, $a_\theta=\log\{\theta/(1-\theta)\}$, $L_-=a_\theta-(1/2)\log(1+\tau^2K_x^2)$, and $L_+=a_\theta+\kappa\lVert y\rVert_2^2K_x^2/(1+\tau^2K_x^2)$. Put $\bar L=\max\{|L_-|,|L_+|\}$ and $\Lambda=\tau^2K_x+2\kappa\lVert y\rVert_2^2K_x(1+\tau^2K_x^2)$. For $\gamma',\gamma''\in\mathcal U_h$, set $d_h(\gamma',\gamma'')=|\gamma'\triangle\gamma''|\bar L$, with $\triangle$ denoting symmetric difference. If their cardinalities agree, write $\gamma'\setminus\gamma''=\{i_1,\ldots,i_d\}$ and $\gamma''\setminus\gamma'=\{i'_1,\ldots,i'_d\}$ in fixed local order, and replace $d_h$ by the smaller of its current value and $\Lambda\sum_{j=1}^d\lVert x_{i_j}-x_{i'_j}\rVert_2$. Then $\rho_4=\sum_h\{\max_{\gamma',\gamma''\in\mathcal U_h}d_h(\gamma',\gamma'')+\osc_{\gamma'\in\mathcal U_h}\log q_{vh}(\gamma')\}$. Set $\bar\rho_v$ to the minimum of $\rho_1$, $\rho_3$, $\rho_4$, and the $\rho_2$ values for the prescribed partitions. For numerical normalization on each allowed set, subtract the largest local log weight, exponentiate, floor at the smallest positive normal double, and renormalize. If the resulting probabilities are $\widetilde q_{vh}$, replace $q_v$ by their product in $g_v$ and sampling, add $\sum_h\osc_{\gamma'\in\mathcal U_h}\log\{q_{vh}(\gamma')/\widetilde q_{vh}(\gamma')\}$ to each matrix bound, and compute $\rho_4$ with $\widetilde q_{vh}$. Spectral calculations use numerical padding and remain subject to the validity qualification in Section~\ref{sec:computation}.

\emph{Direct mass bound.} For a region, its predictor pool $P=\bigcup_h\bigcup_{\gamma'\in\mathcal U_h}\gamma'$ contains every predictor allowed by any of its supports. In a given group order, let $P_h$ contain all allowed predictors in preceding groups and set $D_h(\gamma')=\det(I_{|\gamma'|}+\tau^2X_{\gamma'}^\top C_{P_h}^{-1}X_{\gamma'})$. Compute $c_\sigma\exp\{-y^\top C_P^{-1}y/(2\sigma^2)\}\prod_h\sum_{\gamma'\in\mathcal U_h}\pi_h(\gamma')/\sqrt{D_h(\gamma')}$ in forward and reverse order, and intersect the anchor interval with the upper bound given by their minimum. This bounds the summed regional weight without enumerating its joint supports. Charge the score at $P$ when not cached. If its evaluation exceeds the remaining score allowance, use $+\infty$ for this additional bound. The same determinant factors are used cell by cell in Phase 2.

\subsubsection{Derivation of the proposal and interaction bounds}
\label{app:bma-interaction-derivation}

The Phase 1 proposal approximates a support's weight by multiplying local group weights. Interactions between groups determine how this approximation differs from the joint weight. We derive the interaction term and bound its variation over a region. With $\theta$ and $\kappa$ as above, put $\pi(\gamma)=\theta^{|\gamma|}(1-\theta)^{p-|\gamma|}$ for the support prior, $A_0({\gamma'})=I+\tau^2X_{\gamma'}^\top X_{\gamma'}$, and $b_0({\gamma'})=X_{\gamma'}^\top y$. The determinant lemma and Woodbury identity give $s({\gamma})=c_0\pi({\gamma})e^{\ell_0({\gamma})}$, where $c_0=(2\pi\sigma^2)^{-n/2}e^{-\lVert y\rVert ^2/(2\sigma^2)}$ and
\[
 \ell_0({\gamma'})=-\tfrac12\log\det A_0({\gamma'})+\kappa b_0({\gamma'})^\top A_0({\gamma'})^{-1}b_0({\gamma'}).
 \tag{B.1}\label{eq:bma-local}
\]
Using the local weights and allowed sets from Appendix~\ref{app:bma-initial-routines}, the prior factorizes as $\pi(\gamma)=\prod_h\pi_h(\gamma\cap G_h)$, and the product proposal is
\[
 q_v({\gamma})=\prod_h\frac{w_h({\gamma_h})}{\sum_{U\in\mathcal U_{vh}}w_h(U)},\qquad {\gamma_h}={\gamma}\cap G_h\in\mathcal U_{vh}.
 \tag{B.2}\label{eq:bma-proposal}
\]
Sampling and evaluating this proposal use $\sum_h2^{\lvert G_h\rvert }$ table entries.

\begin{lemma}[Proposal cancellation]
\label{lem:bma-cancellation}
Assume \resultcondition the weights use the Gaussian model and group-factorized prior above; and \resultcondition the proposal is \eqref{eq:bma-proposal}. Then
\[
 \log\{s({\gamma})/q_v({\gamma})\}=\text{constant depending on }v+\ell_0({\gamma})-\sum_h\ell_0({\gamma_h}).
 \tag{B.3}\label{eq:bma-cancellation}
\]
If $X_{G_h}^\top X_{G_{h'}}=0$ for $h\ne h'$, the importance weight is constant on every region.
\end{lemma}
\begin{proof}
Substitution cancels every $\pi_h({\gamma_h})$ and gives
\[
  s({\gamma})/q_v({\gamma})=c_0\prod_h\bigl\{\sum_{U\in\mathcal U_{vh}}w_h(U)\bigr\}
       \exp\bigl\{\ell_0({\gamma})-\sum_h\ell_0({\gamma_h})\bigr\}.
\] Under cross-group orthogonality, $A_0({\gamma})$ is block diagonal: its log determinant and quadratic form add over groups, making the exponent zero. A constant importance weight equals its expectation $Z(R_v)$.
\end{proof}

Consequently, in the orthogonal case,
\[
 Z(R_v)=c_0\prod_h\sum_{{\gamma'}\in\mathcal U_{vh}}w_h({\gamma'}),\qquad
 \Pi(R_v\mid D)=\prod_h\frac{\sum_{{\gamma'}\in\mathcal U_{vh}}w_h({\gamma'})}{\sum_{{\gamma'}\subseteq G_h}w_h({\gamma'})}.
 \tag{B.4}\label{eq:bma-factorization}
\]
Use the range-calculation notation from Appendix~\ref{app:bma-initial-routines}, with $\preceq$ denoting positive-semidefinite order. For $\gamma'\subseteq U$, $A_{\gamma'},b_{\gamma'}$ restrict $A_U,b_U$ to the selected indices, and $\mathsf D_{\gamma'}$ is the corresponding principal submatrix of $\mathsf D$. Applying the Gaussian update to $C_{F\cup\gamma'}=C_F+\tau^2X_{\gamma'}X_{\gamma'}^\top$ gives $\ell_0(F\cup\gamma')=\ell_0(F)+\ell_F(\gamma')$. The chosen eigenvalue bounds satisfy $\lambda_-\mathsf D\preceq A_U\preceq\lambda_+\mathsf D$ and $0<\lambda_-\le1\le\lambda_+$.

\begin{lemma}[Spectral range bound]
\label{lem:bma-spectral}
Assume \resultcondition the proposal is \eqref{eq:bma-proposal}; and \resultcondition $\lambda_-,\lambda_+$ obey the bounds above. Then
\[
 \osc_{{\gamma}\in R_v}\log\{s({\gamma})/q_v({\gamma})\}\le
 \Delta-\tfrac12\overline d\log\lambda_-+\kappa \overline Q(\lambda_-^{-1}-\lambda_+^{-1}).
 \tag{B.5}\label{eq:bma-spectral}
\]
The bound allows unrestricted groups and varying support sizes; a region with no unresolved groups has range zero.
\end{lemma}
\begin{proof}
For ${\gamma'}=\bigcup_{h\in\mathcal H}{\gamma_h}$, the nonconstant part of the log weight is
\[
 \sum_h\{\ell_F({\gamma_h})-\ell_0({\gamma_h})\}
 -\tfrac12\log\frac{\det A_{\gamma'}}{\det \mathsf D_{\gamma'}}
 +\kappa b_{\gamma'}^\top(A_{\gamma'}^{-1}-\mathsf D_{\gamma'}^{-1})b_{\gamma'}.
\]
The first term has oscillation at most $\Delta$. Restricting $\lambda_-\mathsf D\preceq A_U\preceq\lambda_+\mathsf D$ to ${\gamma'}$, and applying Fischer's block determinant inequality, gives $\lambda_-^{\lvert {\gamma'}\rvert }\le\det A_{\gamma'}/\det \mathsf D_{\gamma'}\le1$. Hence the determinant term lies in $[0,-(1/2)\overline d\log\lambda_-]$. Fischer's inequality follows by iterating the Schur-complement identity
\[
  \det\begin{pmatrix}B&C\\C^\top&E\end{pmatrix}
  =\det B\det(E-C^\top B^{-1}C)\le\det B\det E.
\]

Inverting the restricted matrix inequalities bounds the quadratic term between $\kappa(\lambda_+^{-1}-1)q({\gamma'})$ and $\kappa(\lambda_-^{-1}-1)q({\gamma'})$, where $q({\gamma'})=b_{\gamma'}^\top \mathsf D_{\gamma'}^{-1}b_{\gamma'}\le \overline Q$. The first coefficient is nonpositive and the second nonnegative, so its uniform range has width at most $\kappa \overline Q(\lambda_-^{-1}-\lambda_+^{-1})$. Adding the three widths proves the claim.
\end{proof}

For a block partition $\mathscr P$, use the supports $\gamma_i$, within-block oscillation $\Delta_{\mathscr P}$, and vector $\chi$ defined in Appendix~\ref{app:bma-initial-routines}. The interaction matrix is
\[
 \mathsf R_{ij}=\max_{{\gamma_i},{\gamma_j}}\left\lVert A_{\gamma_i}^{-1/2}
       (\tau^2X_{\gamma_i}^\top W_FX_{\gamma_j})A_{\gamma_j}^{-1/2}\right\rVert _2,\qquad i\ne j.
 \tag{B.6}\label{eq:bma-coherence}
\]
Its diagonal is zero, empty supports contribute zero, and $\lambda_R=\lambda_{\max}(\mathsf R)$. The maximum ranges over the prescribed allowed block supports.

\begin{lemma}[Blockwise range bound]
\label{lem:bma-blocks}
Assume \resultcondition $\mathsf R$ bounds the block interactions as in \eqref{eq:bma-coherence}; and \resultcondition $\lambda_R<1$. Then a valid log-weight range bound is
\[
 \Delta_{\mathscr P}-\tfrac12\overline d\log(1-\lambda_R)
                  +2\kappa \chi^\top \mathsf R(I-\mathsf R)^{-1}\chi.
 \tag{B.7}\label{eq:bma-block-bound}
\]
For one block the interaction terms vanish. When $\lambda_R\ge1$, this refinement is unused.
\end{lemma}
\begin{proof}
Fix an allowed support and whiten its exact block diagonal $\mathsf D_{\mathscr P,{\gamma'}}=\operatorname{blockdiag}_iA_{\gamma_i}$. Write $\mathsf D_{\mathscr P,{\gamma'}}^{-1/2}A_{\gamma'}\mathsf D_{\mathscr P,{\gamma'}}^{-1/2}=I+E$ and $u=\mathsf D_{\mathscr P,{\gamma'}}^{-1/2}b_{\gamma'}$. Then $\lVert E_{ij}\rVert _2\le \mathsf R_{ij}$ and $\lVert u_i\rVert _2\le \chi_i$. For a block vector $x$ with $t_i=\lVert x_i\rVert _2$, $\lVert (Ex)_i\rVert _2\le(\mathsf R t)_i$, proving $\lVert E\rVert _2\le \lambda_R$. Thus $(1-\lambda_R)^{\lvert {\gamma'}\rvert }\le\det(I+E)\le1$, using Fischer's inequality for the upper endpoint.

For every integer $\ell\ge1$, induction gives $\lVert (E^\ell x)_i\rVert _2\le(\mathsf R^\ell t)_i$, and consequently $\lvert u^\top E^\ell u\rvert \le \chi^\top \mathsf R^\ell \chi$. The convergent Neumann series bounds $\lvert u^\top\{(I+E)^{-1}-I\}u\rvert \le \chi^\top \mathsf R(I-\mathsf R)^{-1}\chi$. Its quadratic contribution therefore has twice this bound times $\kappa$ as its range. Adding the determinant width and the within-block oscillation $\Delta_{\mathscr P}$ proves the result.
\end{proof}

When each unresolved group permits exactly one predictor, recompute $\mathsf R$, $\chi$, and $\lambda_R$ with one group per block. This makes $E$ scalar-blocked and $\lvert E\rvert \le \mathsf R$ entrywise. The absolutely convergent log-determinant series and $\lvert E^\ell\rvert \le \mathsf R^\ell$ for every integer $\ell\ge1$ imply $\log\det(I+E)\ge\log\det(I-\mathsf R)$. The sharper bound is therefore
\[
 \Delta-\tfrac12\log\det(I-\mathsf R)
                     +2\kappa \chi^\top \mathsf R(I-\mathsf R)^{-1}\chi,\qquad \lambda_R<1.
 \tag{B.8}\label{eq:bma-one-bound}
\]
The block choices are fixed in Appendix~\ref{app:bma-initial-routines}. Cross-block maxima inspect pairs of local tables; this matrix work enters elapsed time but not full-support score-call counts. Groups exceeding the configuration cap use the spectral and path bounds.

\subsection{Path bounds and proposal support}
\label{app:bma-path}
\label{app:gaussian-bound}
\label{app:numerical-bounds}

The spectral and block bounds above use matrix structure. A path bound controls changes in log weight under predictor additions, removals, or exchanges, and sums them along a path between supports. It provides a finite alternative for every branch type. Using $x_j$ and $K_x$ from Appendix~\ref{app:bma-initial-routines}, put $M_{\gamma}=(I+\tau^2X_{\gamma}X_{\gamma}^\top)^{-1}$. Adding a column $x$ changes the log likelihood by
$L(x)=-(1/2)\log(1+\tau^2 a_x)+\tau^2b_x^2/\{2\sigma^2(1+\tau^2a_x)\}$,
where $a_x=x^\top M_{\gamma}x$ and $b_x=y^\top M_{\gamma}x$. Since $0\prec M_{\gamma}\preceq I$, for $d=\lVert x-x'\rVert _2$ we have $\lvert a_x-a_{x'}\rvert \le2K_xd$, $\lvert b_x-b_{x'}\rvert \le\lVert y\rVert _2d$, and $\lvert b_x\rvert \le\lVert y\rVert _2K_x$. The determinant term has derivative at most $\tau^2/2$ in absolute value. The quadratic term has partial derivatives bounded by $\tau^4\lVert y\rVert _2^2K_x^2/(2\sigma^2)$ and $\tau^2\lVert y\rVert _2K_x/\sigma^2$. For $i,j\notin\gamma$, integrating along the segment between the scalar arguments gives
\[
\begin{aligned}
 \lvert \log s({\gamma}\cup\{i\})-\log s({\gamma}\cup\{j\})\rvert 
 &\le\lvert \log\pi({\gamma}\cup\{i\})-\log\pi({\gamma}\cup\{j\})\rvert \\
 &\quad+\left(\tau^2K_x+\frac{\tau^2\lVert y\rVert _2^2K_x}{\sigma^2}
                   +\frac{\tau^4\lVert y\rVert _2^2K_x^3}{\sigma^2}\right)\lVert x_i-x_j\rVert _2.
\end{aligned}
\tag{B.9}\label{eq:gaussian-swap}
\]

For an addition $x=x_j$ with $j\notin\gamma$, $0\le a_x\le K_x^2$ and $b_x^2\le\lVert y\rVert _2^2a_x$, by $M_{\gamma}^2\preceq M_{\gamma}$. With $a_\theta=\log\{\theta/(1-\theta)\}$, this gives
\[
 L_-=a_\theta-\tfrac12\log(1+\tau^2K_x^2)
 \le\log\frac{s({\gamma}\cup\{j\})}{s({\gamma})}
 \le a_\theta+\frac{\tau^2\lVert y\rVert _2^2K_x^2}{2\sigma^2(1+\tau^2K_x^2)}=L_+.
 \tag{B.10}\label{eq:bma-add}
\]
The upper bound drops the nonpositive log term and uses monotonicity of $a/(1+\tau^2a)$, while the lower drops the nonnegative quadratic term. Deletions negate additions. These increments justify the prescribed distance $d_h$ in Appendix~\ref{app:bma-initial-routines}: use additions and deletions, or the smaller swap bound when cardinalities agree. Under the Bernoulli prior the swap prior term vanishes.

Changing groups successively and telescoping the increments bounds the log-weight difference by the sum of these distances. Intermediate supports may leave the region because the increments hold throughout the full support space. The proposal contributes at most each local table's log oscillation. Hence
\[
 \sum_h\left\{\max_{{\gamma'},{\gamma''}\in\mathcal U_{vh}}d_h({\gamma'},{\gamma''})
                              +\osc_{{\gamma'}\in\mathcal U_{vh}}\log q_{vh}({\gamma'})\right\}
 \tag{B.11}\label{eq:bma-path}
\]
bounds the regional log importance-weight range. For node $v$, collect the bounds from \eqref{eq:bma-spectral}, \eqref{eq:bma-one-bound}, and \eqref{eq:bma-path}, together with the values of \eqref{eq:bma-block-bound} for the three block constructions above, in a set $\mathscr B_v$. Assign $+\infty$ to an inapplicable bound. The implementation uses
\[
 \bar\rho_v=\min\mathscr B_v,
 \tag{B.12}\label{eq:bma-production}
\]
and sets $\bar\rho_v=0$ for a singleton region. Every term in $\mathscr B_v$ bounds the same range.

The implementation conditions local log factors on the allowed configurations before normalization. If support preservation changes ideal probabilities $q_{vh}$ to positive probabilities $\widetilde q_{vh}$, then $\log(s/\widetilde q_v)=\log(s/q_v)+\sum_h\log(q_{vh}/\widetilde q_{vh})$. Adding $\sum_h\osc\log(q_{vh}/\widetilde q_{vh})$ to the ideal range bound accounts for this change. Local normalizing constants cancel from each oscillation. Cholesky calculations and padded spectral endpoints remain subject to the floating-point qualification in Section~\ref{sec:computation}.

\subsection{Gaussian envelopes and their optimization}
\label{app:bma-envelopes}
\label{app:envelope-convexity}

\subsubsection{Envelope construction}
\label{app:bma-envelope-routine}

\noindent\textit{BMA refinement envelopes.} The three steps below implement \textsc{PrepareEnvelope} in Phase 2. The computational cells are used only to construct and sample an envelope within a terminal region.

\looseness=-1
\emph{Construct cell envelopes.} Start with one computational cell $E_1=R_v$. In supplied group order, split every cell by each allowed configuration of the next group, stopping before a split would exceed 256 cells. The resulting cells $E_c$ partition $R_v$ for sampling and do not change its tree node. Within cell $E_c$, write $\mathcal U_h$ for its allowed local supports and recompute $P$, $P_h$, and $D_h$ from those sets. Choose the forward or reverse order giving the smaller product $\prod_h\sum_{{\gamma'}\in\mathcal U_h}\pi_h({\gamma'})/\sqrt{D_h({\gamma'})}$, preferring forward on a tie. An auxiliary vector $z\in\mathbb R^n$ controls the envelope shape. Set $c(z)=c_\sigma\exp\{(\lVert z\rVert_2^2-2y^\top z)/(2\sigma^2)\}$ and $e_{z,h}({\gamma'})=\pi_h({\gamma'})\exp\{\tau^2\lVert X_{\gamma'}^\top z\rVert_2^2/(2\sigma^2)\}/\sqrt{D_h({\gamma'})}$. Their product $e_z(\gamma)=c(z)\prod_he_{z,h}(\gamma\cap G_h)$ bounds $s(\gamma)$ on $E_c$ for every fixed $z$. Summing the group factors gives $U(z)=c(z)\prod_h\sum_{{\gamma'}\in\mathcal U_h}e_{z,h}({\gamma'})$ and $q_z=e_z/U(z)$.

\emph{Tighten each envelope.} Minimize $F(z)=\log U(z)$ on $y+\operatorname{span}(X_P)$, starting from $z^{[0]}=C_P^{-1}y$. Use the Newton direction $d$, replacing a failed or non-descent direction by $-\sigma^2$ times the projected gradient. Accept the first $\lambda\in\{1,1/2,\ldots,2^{-40}\}$ satisfying $F(z+\lambda d)\le F(z)+10^{-4}\lambda\nabla F(z)^\top d$. Stop when $\sigma^2\lVert\nabla F(z)\rVert_2^2/2\le10^{-6}$, after 60 iterations, or when no step is accepted. Retain the last accepted vector as $z_c$, or $z^{[0]}$ if none is accepted. Every retained vector gives a valid analytical envelope.

\looseness=-1
\emph{Combine and normalize.} Let $\mathcal U_{ch}$ be the allowed local supports in $E_c$. Normalize $e_{z_c,h}$ by the positive-floor rule in Appendix~\ref{app:bma-initial-routines} to obtain $\widetilde q_{ch}$. Set $q_c(\gamma)=\prod_h\widetilde q_{ch}(\gamma\cap G_h)$ and $U_c=c(z_c)\prod_h\max_{\gamma'\in\mathcal U_{ch}}\{e_{z_c,h}(\gamma')/\widetilde q_{ch}(\gamma')\}$. Add $10^{-9}\max\{1,|\log U_c|\}$ to the computed cell log bound. Choose cell probabilities proportional to these padded bounds, floor at the smallest positive normal double, and renormalize to $p_c$. Then $q_v(\gamma)=p_cq_c(\gamma)$ on $E_c$ and $U_v=\max_cU_c/p_c$. Add $10^{-9}\max\{1,|\log\sum_cU_c|\}$ to the regional log bound. The prepared pair $(q_v,U_v)$ stays fixed throughout the leaf's refinement stream.

\subsubsection{Validity and optimization properties}
\label{app:bma-envelope-validity}

Fix a Cartesian cell $E_c\subseteq R_v$ and write $\mathcal U_h$ for its allowed local supports. Use the factors $D_h({\gamma'}),c(z),e_{z,h}({\gamma'})$ defined in Appendices~\ref{app:bma-initial-routines} and \ref{app:bma-envelope-routine}. For a given support, the predictor indices ${\gamma}_{<h}$ selected from preceding groups satisfy ${\gamma}_{<h}\subseteq P_h$. Thus $C_{{\gamma}_{<h}}^{-1}\succeq C_{P_h}^{-1}$. Applying the determinant lemma at each group and multiplying proves $\det C_{\gamma}\ge\prod_hD_h({\gamma_h})$. Also $(y-C_{\gamma}z)^\top C_{\gamma}^{-1}(y-C_{\gamma}z)\ge0$ proves the quadratic bound used there. Since $z^\top C_{\gamma}z=\lVert z\rVert ^2+\tau^2\sum_h\lVert X_{\gamma_h}^\top z\rVert ^2$, substituting both inequalities gives $s({\gamma})\le c(z)\prod_he_{z,h}({\gamma_h})$. Summation over the Cartesian product yields $U(z)$ and the normalized positive proposal $q_z$. This proves both boundedness and the mean identity for the refinement weights.

The direct mass bound used during exploration takes $P=\bigcup_h\bigcup_{{\gamma'}\in\mathcal U_h}{\gamma'}$. Since ${\gamma}\subseteq P$, $y^\top C_{\gamma}^{-1}y\ge y^\top C_P^{-1}y$, and
\[
 Z(E_c)\le c_\sigma e^{-y^\top C_P^{-1}y/(2\sigma^2)}
                   \prod_h\sum_{{\gamma'}\in\mathcal U_h}\frac{\pi_h({\gamma'})}{\sqrt{D_h({\gamma'})}}.
\]
The support $P$ need not itself meet the region's count restrictions.

For optimization, put $d_h({\gamma'})=\pi_h({\gamma'})/\sqrt{D_h({\gamma'})}$ and $M_{h,{\gamma'}}=X_{\gamma'}X_{\gamma'}^\top$. The exact objective $F(z)=\log U(z)$ is
\[
 F(z)=\log c_\sigma+\frac{\lVert z\rVert ^2-2y^\top z}{2\sigma^2}
       +\sum_h\log\bigl\{\sum_{{\gamma'}\in\mathcal U_h}d_h({\gamma'})
                        e^{\tau^2z^\top M_{h,{\gamma'}}z/(2\sigma^2)}\bigr\}.
\]
For $\gamma'\in\mathcal U_h$, the local probability is $\nu_{h,z}(\gamma')=e_{z,h}(\gamma')/\sum_{U\in\mathcal U_h}e_{z,h}(U)$.

\begin{proposition}[Convex Gaussian envelope preparation]
\label{prop:envelope-convexity}
Assume \resultcondition $\sigma^2,\tau^2>0$; and \resultcondition each $\mathcal U_h$ is finite and nonempty. Then the following conclusions hold.
\begin{enumerate}[label=(\alph*),leftmargin=*,itemsep=2pt,topsep=3pt]
\item $F$ is globally $\sigma^{-2}$-strongly convex with unique minimizer $z^*$. Its gradient and minimizer satisfy
\[
 \nabla F(z)=\frac{z-y+\tau^2\sum_h\E_{\nu_{h,z}}(M_{h,{\gamma'}}z)}{\sigma^2},\quad
 \bigl[I+\tau^2\sum_h\E_{\nu_{h,z^*}}M_{h,{\gamma'}}\bigr]z^*=y,
\]
with $\lVert z^*\rVert _2\le\lVert y\rVert _2$ and $z^*-y\in\operatorname{span}(X_P)$.
\item For every $z$,
$0\le F(z)-F(z^*)\le(\sigma^2/2)\lVert \nabla F(z)\rVert _2^2$.
\end{enumerate}
\end{proposition}
\begin{proof}
Differentiating the finite positive sums gives the stated gradient and
\[
 \nabla^2F(z)=\sigma^{-2}I
 +\tau^2\sigma^{-2}\sum_h\E_{\nu_{h,z}}M_{h,{\gamma'}}
 +\tau^4\sigma^{-4}\sum_h\operatorname{Cov}_{\nu_{h,z}}(M_{h,{\gamma'}}z)
 \succeq\sigma^{-2}I.
\]
Strong convexity implies coercivity and a unique minimum. Setting the gradient to zero gives the stationarity equation and the span claim. Multiplication by $z^{*\top}$ gives
$\lVert z^*\rVert ^2+\tau^2\sum_h\E\lVert X_{\gamma'}^\top z^*\rVert ^2 =z^{*\top}y\le\lVert z^*\rVert\,\lVert y\rVert$,
proving the norm bound. Finally, minimize the strong-convexity lower bound $F(z)+\nabla F(z)^\top(b-z)+{\lVert b-z\rVert ^2}/(2\sigma^2)$ over $b$ to obtain the error bound.
\end{proof}

An orthonormal basis of $\operatorname{span}(X_P)$ permits optimization in $\operatorname{rank}(X_P)$ dimensions. The implementation uses an untruncated QR matrix $Q$ whose orthonormal columns span a space containing $\operatorname{span}(X_P)$. It applies the Newton rule in Appendix~\ref{app:bma-envelope-routine} to $\phi(b)=F(y+Qb)$, retaining the optimum even if extra directions are kept. The bound $U(z)/U(z^*)\le e^{\sigma^2\lVert\nabla F(z)\rVert^2/2}$ controls optimization within this envelope family, not $U(z^*)/Z(E_c)$. The component study's BFGS variant uses at most 100 iterations, relative tolerance $10^{-10}$, and retains the starting vector if its objective is smaller than the final candidate's.

For the positive local probabilities $\widetilde q_h$ actually used, each ratio $e_{z,h}(\gamma')/\widetilde q_h(\gamma')$ is bounded by its local maximum. Multiplying gives the corrected envelope
$\log\widetilde U(z)=\log c(z)+\sum_h\max_{\gamma'\in\mathcal U_h}\{\log e_{z,h}(\gamma')-\log\widetilde q_h(\gamma')\}$.
Ideal normalization recovers $U(z)$. For disjoint cells, $s(\gamma)/q_v(\gamma)\le U_c/p_c\le U_v$ on $E_c$, proving validity after cell-probability adjustment. Ideally $p_c=U_c/\sum_bU_b$ and $U_v=\sum_bU_b$. Appendix~\ref{app:bma-envelope-routine} states the normalization, padding, and cell-splitting rules. Their numerical scope is specified in Section~\ref{sec:computation}.

\subsection{Selection, allocation, and mass-conservation updates}
\label{app:optimization}
\label{app:implementation}
\label{app:mass-propagation}

Table~\ref{tab:core-notation} distinguishes current and saved quantities. The selection recurrence, numerical settings, and mass-conservation proof below support Algorithm~\ref{alg:workflow}.

\begin{table}[!tb]
\centering
\caption{Current and saved quantities in Algorithm~\ref{alg:workflow}.}
\label{tab:core-notation}
\small
\renewcommand{\arraystretch}{1.08}
\begin{tabular*}{\textwidth}{@{\extracolsep{\fill}}L{.27\textwidth}L{\dimexpr.73\textwidth-2\tabcolsep\relax}@{}}
\toprule
Symbol & Meaning\\
\midrule
$\mathcal V^{(k)},\mathcal F^{(k)},\mathcal E^{(k)}$ & Recorded nodes, frontier, and eligible leaves at update $k$.\\
$Z_v,\ B_v^{(k)}$ & Fixed regional constant and its stored interval at update $k$.\\
$J^{(k)},\ J_k$ & Current and saved leaf selections, with $A_J=\bigcup_{v\in J}R_v$.\\
$S^{(k)},\ \mathcal S^{(k)}(A)$ & Optimistic optimizer and frontier nodes intersecting $A$.\\
$\underline C^{(k)},\overline C^{(k)}$ & Complexity bounds from the current selection problems.\\
$\underline C_k,\overline C_k$ & Strongest saved lower and upper complexity bounds.\\
$\zeta^{(k)},\zeta_k,\mathcal C_k$ & Current/saved omitted-mass bounds and record $\mathcal C_k=(J_k,\underline C_k,\overline C_k,\zeta_k)$.\\
$t_v^{(k)},\mathcal W_{v,t}$ & Leaf $v$'s checkpoint count at update $k$, and its weight vector at checkpoint $t$.\\
\bottomrule
\end{tabular*}
\end{table}

\noindent\textit{Solving the selection problems.} At selection update $k$, write $l_v^{(k)}=\underline Z_v^{(k)}$ and $u_v^{(k)}=\overline Z_v^{(k)}$. The optimistic selection problem has item benefits $(1-\epsilon)l_v^{(k)}+\epsilon u_v^{(k)}$, complexities $\underline c_v$, and covering threshold $(1-\epsilon)\sum_{v\in\mathcal F^{(k)}}l_v^{(k)}$. The conservative problem has eligible-leaf benefits $(1-\epsilon)u_v^{(k)}+\epsilon l_v^{(k)}$, complexities $c_v$, and threshold $(1-\epsilon)\sum_{v\in\mathcal F^{(k)}}u_v^{(k)}$. These forms follow from \eqref{eq:frontier-programs} by rearranging the mass inequalities and collecting selected terms.

At update $k$, the item set is $\mathcal E^{(k)}$ for the conservative problem and $\mathcal F^{(k)}$ for the optimistic problem. Enumerate those items by $j$ with benefit $b_j^{(k)}$ and complexity $c_j^{(k)}$. The latter reflects the current item mapping, not a changing regional complexity. For positive integer weights, the largest benefit at budget $c$ among the first $j$ items satisfies
\[
  V_j^{(k)}(c)=\max\{V_{j-1}^{(k)}(c),b_j^{(k)}+V_{j-1}^{(k)}(c-c_j^{(k)})\},
\]
omitting the second term if $c<c_j^{(k)}$ and initializing $V_0^{(k)}(c)=0$. Excluding or including the last item proves the recurrence by induction. The first budget reaching the threshold is optimal. At threshold zero, choose a least complex singleton from that problem's item set to enforce nonemptiness. An empty item set means infeasibility. Items follow frontier order, and an equal-benefit inclusion retains the previous entry. Backtracking stores item-by-budget states with complexity-axis size $1+\sum_j c_j^{(k)}$, so the solver is pseudo-polynomial in the complexity weights. The BMA terminal family grows exponentially with the number of groups. Directed MPFR arithmetic encloses coefficients derived from supplied log endpoints and verifies the original mass inequality. The directed-arithmetic solver starts at 128 bits or the log-endpoint span in binary digits plus 96 guard bits, whichever is larger, capped at 16,384 bits. It doubles precision when unresolved up to that cap. Before dynamic programming, it verifies candidates from the full item set and prefixes ordered by decreasing benefit and benefit per complexity, with frontier order for ties, retaining a prefix only if its complexity is smaller. If the declared 192 MiB table limit is exceeded, it returns that verified candidate and the smallest item complexity as a valid lower bound. At the precision limit it retains the computed lower bound and verified candidate, without claiming an exact optimum.

\noindent\textit{Implementation settings.} Section~\ref{sec:adaptive-updates} defines the scheduling scores and evaluation rules. The implemented regional complexities are positive. The \textsc{CompleteLeaf} limits are stated in Appendix~\ref{app:bma-initial-routines}, with initial relative-error target $\eta_0=.20$ in the BMA comparison. For $0<m_N<1$, KL endpoints use the outer brackets after 60 bisections, padded outward by $10^{-12}$. At $m_N=0$ or 1, the closed-form endpoint limits are used. Refinement uses $T=6$ checkpoints with $(N_1,\ldots,N_6)=(256,512,1024,2048,4096,8192)$ and the failure allowances in Appendix~\ref{app:sampling}.

Computation includes regional preparation, score evaluation, tree expansion, and repeated selection. Gram matrices, local tables, determinant factors, optimizer starts, and predictor-span bases are cached within a dataset, with no cache across datasets. Time is checked at action boundaries, so one atomic update may exceed the allowance.

\noindent\textit{Reusing parent and child intervals.} For a parent $v$ with complete child set $\mathcal H(v)$, mass conservation gives $Z_v=\sum_{w\in\mathcal H(v)}Z_w$. Within a fixed selection update $k$, let $l_{v,j}^{(k)},u_{v,j}^{(k)}$ be the endpoints after $j$ local tightening operations. These start from the provisional intervals and end at $\underline Z_v^{(k)},\overline Z_v^{(k)}$; $j$ does not advance the search index $k$. All sums below are over these children.

\begin{proposition}[Mass-conservation tightening]
\label{prop:mass-propagation}
Assume \resultcondition the intervals $[l_{v,j}^{(k)},u_{v,j}^{(k)}]$ are finite and simultaneously valid; and \resultcondition every expanded parent's complete child set is included. Each local operation changes one coordinate and copies the others from $j$ to $j+1$. The updates
\[
\begin{aligned}
 l_{w,j+1}^{(k)}&=\max\bigl\{l_{w,j}^{(k)},l_{v,j}^{(k)}-\sum_{w'\ne w}u_{w',j}^{(k)}\bigr\},&
 u_{w,j+1}^{(k)}&=\min\bigl\{u_{w,j}^{(k)},u_{v,j}^{(k)}-\sum_{w'\ne w}l_{w',j}^{(k)}\bigr\},\\
 l_{v,j+1}^{(k)}&=\max\bigl\{l_{v,j}^{(k)},\sum_w l_{w,j}^{(k)}\bigr\},&
 u_{v,j+1}^{(k)}&=\min\bigl\{u_{v,j}^{(k)},\sum_w u_{w,j}^{(k)}\bigr\}
\end{aligned}
\]
have the following properties.
\begin{enumerate}[label=(\alph*),leftmargin=*,itemsep=2pt,topsep=3pt]
\item They preserve validity and only narrow intervals. Any finite, possibly adaptive sequence needs no additional failure allowance.
\item With the frontier, eligibility, and complexities fixed, the optimistic complexity bound cannot decrease and the conservative complexity bound cannot increase.
\end{enumerate}
\end{proposition}
\begin{proof}
The identity gives
\[
l_{v,j}^{(k)}-\sum_{w'\ne w}u_{w',j}^{(k)}\le Z_w\le u_{v,j}^{(k)}-\sum_{w'\ne w}l_{w',j}^{(k)},\qquad
 \sum_w l_{w,j}^{(k)}\le Z_v\le\sum_w u_{w,j}^{(k)}.
\] Intersecting preserves each true value, and induction proves validity for every finite update sequence on the original event. For a fixed optimistic candidate, the signed margin $\epsilon\sum_{v\in S}u_{v,j}^{(k)}-(1-\epsilon)\sum_{v\notin S}l_{v,j}^{(k)}$ cannot increase, so its feasible family only shrinks. The conservative margin $\epsilon\sum_{v\in J}l_{v,j}^{(k)}-(1-\epsilon)\sum_{v\notin J}u_{v,j}^{(k)}$ cannot decrease, so its feasible family only expands. Minimum fixed complexities consequently have the claimed monotonicity.
\end{proof}

The main BMA comparison uses the four-pass rule in Section~\ref{sec:adaptive-updates}, with directed rounding and early stopping when no endpoint changes. Parents are visited in reverse creation order bottom-up and creation order top-down. Conflicts terminate the search. The component and Tecator protocols explicitly disable propagation, as stated with their designs. Propagation adds arithmetic but uses no scores or sampling allowance. The proposition requires neither a fixed point nor a finite convergence rate for indefinite iteration.

\FloatBarrier

\section{Supplementary empirical evidence}
\label{app:experiments}

This appendix provides the experimental specifications and additional outcomes supporting Sections~\ref{sec:experiments} and \ref{sec:real-data}. It covers the regression comparisons, component diagnostics, the network simulation, and the Tecator and Sachs applications.

\subsection{BMA comparison design}
\label{app:bma-scaling-design}

The 200-dataset design, posterior, regional complexities, tolerances, and budgets are specified in Section~\ref{sec:experiment-design}. Let $b$ denote the between-group population correlation. For observation $i$ and predictor $j$ in group $h$, generate unstandardized entries
\[
  \widetilde X_{ij}=\sqrt b\,F_i+\sqrt{.95-b}\,F_{ih}+\sqrt{.05}\,E_{ij}
\]
from independent standard-normal variables, then center and scale each column to sample variance one to obtain $X$. Generate $y=X\beta+\eta$ with independent $\eta\sim N(0,I_{80})$ and center the response. The main coefficients $(1.4,-1.12,.98)$ act on the first predictors of the first three groups. Stress coefficients proportional to $(1.4,-1.12,.98,-.84,.70)$ act on the first predictors of five groups, with a common scale matching the main design's population signal variance at the same $b$. Ten independent datasets in each of 16 main and four stress conditions give the stated totals; increasing dimension adds null groups. The reproducibility supplement records the data and algorithm seeds in each dataset's \texttt{data\_manifest.json}.

ARS uses safeguarded Newton optimization, adaptive allocation, mass conservation, direct summation up to 256 supports, and the six-checkpoint grid in Appendix~\ref{app:implementation}. The independent MC$^3$ implementation shares the numerical support scorer but uses its own sampler and regional selector. Each transition chooses a single-variable toggle or active--inactive swap with equal probability and uses the Metropolis acceptance rule; both proposals are symmetric. Swaps at empty or full supports are self-transitions. Four chains start at the empty support, full support, and independent Bernoulli supports with inclusion probabilities $.1$ and $.5$, then advance in round-robin order. Frequency estimates discard the first quarter of each chain's planned transitions. Visited-score estimates include distinct visited supports from burn-in. An independent selector for integer complexities finds a union with estimated mass at least $.90$, with a recorded greedy fallback if its memory limit is reached. The two estimators share one search. Across both request ceilings, there are 800 primary searches and 1,200 estimator outcomes.

\noindent\textit{Recorded elapsed times.} At the 100,000-request ceiling, median elapsed times in the 160 main-design datasets were 49.56 seconds for ARS and 3.24 seconds for the shared MC$^3$ search. In the 40 diffuse-signal stress datasets, the corresponding medians were 131.49 and 5.85 seconds. Medians include all attempts, including searches that return no summary. Both MC$^3$ summaries share one search time, including construction of both summaries. ARS takes substantially longer in these comparisons. The methods can stop at different times and perform different amounts of matrix and selection work per score request. Performance at matched elapsed time remains untested. Table~\ref{tab:tecator-ars-all} reports the recorded times for the six Tecator settings.

Each practical search runs in a fresh R process, with method order deterministically rotated and one practical worker at a time. The recorded execution environment was Windows 11 (build 26300), with R 4.4.2 and a Python 3.12.14 controller. Processor and RAM specifications were not captured in the execution record. All score requests, including preparation and cache hits, count toward the budget. Timing starts before score-cache setup and includes preparation, tree operations, sampling, and selection; data loading, interpreter startup, compilation, and auditing are excluded. The 300-second allowance is checked at action boundaries, with an additional 420-second worker limit. The machine was not reserved exclusively for benchmarking. Component timings can overlap and are not additive.

After all practical methods finish on a dataset, an independent worker enumerates its $2^p$ supports, checks support scores, accumulates regional masses, solves the reference selection problem, and audits saved ARS bounds and checkpoints. Practical workers do not read reference files. All failed and incomplete searches remain in the denominators. Retained-mass distributions condition on returned summaries, complexity ratios on audited feasibility, and paired complexities on joint feasibility. The 95\% paired percentile bootstrap intervals use the 2.5th and 97.5th percentiles from 2,000 dataset resamples stratified by condition, retaining paired method outcomes together.

\begin{table}[!htb]
\centering
\caption{Additional BMA outcomes. Panel (a) reports the diffuse-signal stress comparison at the 100,000-request ceiling and 300-second allowance. Panel (b) reports MC$^3$ mass shortfalls at both request ceilings.}
\label{tab:bma-mass-shortfalls}
\small
\setlength{\tabcolsep}{5pt}
\begin{tabular*}{\textwidth}{@{\extracolsep{\fill}}lrrrr@{}}
\toprule
\multicolumn{5}{l}{\textit{(a) Diffuse-signal stress comparison}}\\
Method & Feasible & \shortstack{Optimum\\attained} & \shortstack{Optimum\\certified} & \shortstack{Median\\complexity ratio}\\
\midrule
\multicolumn{5}{l}{Diffuse-signal stress design, 40 datasets}\\
ARS & 40 & 7 & 0 & 1.61\\
MC$^3$ frequencies & 38 & 29 & --- & 1.00\\
MC$^3$ visited scores & 3 & 3 & --- & 1.00\\
\bottomrule
\end{tabular*}
\par\smallskip
\begin{minipage}{\textwidth}\small
\textit{Notes for panel (a).} Feasibility and attained optimality are assessed by exhaustive post-search enumeration. Complexity-ratio medians use feasible outputs. Certification is conditional on valid computed regional bounds (Section~\ref{sec:computation}).
\end{minipage}
\par\medskip
\begin{tabular*}{\textwidth}{@{\extracolsep{\fill}}lrlrrr@{}}
\toprule
\multicolumn{6}{l}{\textit{(b) Mass shortfalls in MC$^3$ summaries}}\\
Family & Requests & Estimator & Misses & Median & Maximum\\
\midrule
Main & 25,000 & Frequencies & 56/160 & .0079 & .0356\\
Main & 25,000 & Visited scores & 97/160 & .0170 & .0837\\
Main & 100,000 & Frequencies & 37/160 & .0046 & .0170\\
Main & 100,000 & Visited scores & 67/160 & .0078 & .0441\\
Stress & 25,000 & Frequencies & 16/40 & .0055 & .0309\\
Stress & 25,000 & Visited scores & 40/40 & .0480 & .1341\\
Stress & 100,000 & Frequencies & 2/40 & .0060 & .0105\\
Stress & 100,000 & Visited scores & 37/40 & .0337 & .0898\\
\bottomrule
\end{tabular*}
\end{table}

\noindent\textit{Diffuse-signal stress comparison.} The 40 diffuse-signal stress datasets show the limits of finding compact summaries within these budgets. ARS met the mass target on every dataset at both request ceilings. Increasing the ceiling from 25,000 to 100,000 reduced its median complexity ratio from 6.01 to 1.61, but only seven outputs attained the optimum and none certified it. At 100,000 requests, MC$^3$ frequencies met the target on 38 datasets and attained 29 optima. Its two mass shortfalls were $.0014$ and $.0105$. MC$^3$ visited scores met the target on only three datasets, so its median complexity ratio in Table~\ref{tab:bma-mass-shortfalls}(a) describes just those three outputs. In this design, ARS retained sufficient mass more consistently, while frequency-based MC$^3$ more often found a minimum-complexity summary.

Panel (a) of Table~\ref{tab:bma-mass-shortfalls} reports the stress-design mass and complexity outcomes at the larger request ceiling. A dash indicates that MC$^3$ does not supply an optimality certificate.

\noindent\textit{Additional outcomes.} At 25,000 requests, ARS returns feasible summaries on 148 of 160 main datasets, attains 91 optima, and certifies 86. Frequency and visited-score summaries are feasible on 104 and 63 datasets. The twelve ARS nonreturns occur at 24 predictors in pairs. In the 40 stress datasets the corresponding feasible counts are 40, 24, and zero. Table~\ref{tab:bma-mass-shortfalls}(b) retains all MC$^3$ mass misses at both ceilings.

Panel (b) of Table~\ref{tab:bma-mass-shortfalls} counts returned MC$^3$ summaries below the $.90$ mass target out of all datasets in each family. Median and maximum shortfalls are on the probability scale and condition on a miss. The two estimators share each sampling run.

At 100,000 requests, the twelve main-design nonreturns all occur at 24 predictors in pairs and reach the time allowance during exploration. Among the 117 main datasets jointly feasible for ARS and frequency-based MC$^3$, the median paired complexity ratio is one.

\FloatBarrier

\subsection{Component comparisons and search diagnostics}
\label{app:bma-confirmatory-details}
\label{app:search-diagnostics}
\label{app:matched-ranking}

\noindent\textit{Primary design and outcomes.} The component experiments cross 24 settings with 100 independent datasets each and four tolerances $\epsilon\in\{.01,.05,.10,.20\}$, giving 9,600 primary searches. ARS uses $\delta=.05$, 25,000 score requests, 60 seconds, forward group order, KL-refined intervals, adaptive allocation, and the same group-count complexities as the main comparison. Envelope optimization uses BFGS, and mass-conservation propagation is disabled. Since complexity is integer-valued, the requested gap $.999999$ certifies optimality. Paired comparisons vary individual components on subsets of these datasets. Repeated targets and variants are not independent replications.

Random designs use the factor construction in Appendix~\ref{app:bma-scaling-design}, also allowing $b=0$, which does not impose empirical orthogonality. Orthogonal controls instead use $X=\sqrt{n-1}\,QL^\top$, with centered orthonormal $Q$ and block-diagonal $LL^\top$ having paired correlation $.95$. All settings use $n=80$, noise variance one, $\tau^2=2$, and inclusion probability $.1$. Main paired designs activate groups 1, 3, and 5; size comparisons activate groups 1, 2, and 3, with sizes $(2,2,2,2,2,2)$, $(3,3,3,3)$, $(4,4,4)$, or $(2,3,3,4)$. Coefficients are $(1.4,-1.12,.98)$. Two-active designs split the first effect into $(.7,.7)$ within its pair and rescale to preserve signal variance. Quartet stress settings use five active groups and the scaling above.

\begin{table}[!tbp]
\centering
\caption{Feasibility and certification across regression settings. Entries give feasible/certified counts. Each setting uses 100 independent datasets.}
\label{tab:bma-confirmatory-all}
\small
\setlength{\tabcolsep}{6pt}
\renewcommand{\arraystretch}{1.06}
\begin{tabular*}{\textwidth}{@{\extracolsep{\fill}}rlrcrrrr@{}}
\toprule
Case & Design & $p$ & $b$ & $\epsilon=.01$ & $.05$ & $.10$ & $.20$\\
\midrule
1 & Pairs, control & 12 & orth. & 100/100 & 100/100 & 100/100 & 100/100\\
2 & Pairs, main & 12 & 0 & 100/100 & 100/100 & 100/100 & 100/100\\
3 & Pairs, main & 12 & .05 & 100/100 & 100/100 & 100/100 & 100/100\\
4 & Pairs, main & 12 & .20 & 100/100 & 100/100 & 100/100 & 100/100\\
5 & Pairs, control & 16 & orth. & 100/100 & 100/100 & 100/100 & 100/100\\
6 & Pairs, main & 16 & 0 & 100/52 & 100/89 & 100/97 & 100/100\\
7 & Pairs, main & 16 & .05 & 100/52 & 100/94 & 100/97 & 100/98\\
8 & Pairs, main & 16 & .20 & 99/43 & 100/92 & 100/94 & 100/97\\
9 & Pairs, control & 20 & orth. & 100/100 & 100/100 & 100/100 & 100/100\\
10 & Pairs, main & 20 & 0 & 0/0 & 28/1 & 51/5 & 71/7\\
11 & Pairs, main & 20 & .05 & 0/0 & 28/0 & 52/2 & 68/4\\
12 & Pairs, main & 20 & .20 & 0/0 & 28/0 & 47/0 & 64/0\\
13 & Pairs, size & 12 & .05 & 100/100 & 100/100 & 100/100 & 100/100\\
14 & Triples & 12 & .05 & 100/100 & 100/100 & 100/100 & 100/100\\
15 & Quartets & 12 & .05 & 100/88 & 100/97 & 100/84 & 100/96\\
16 & Mixed & 12 & .05 & 100/96 & 100/100 & 100/100 & 100/100\\
17 & Pairs, two active & 16 & .05 & 99/49 & 100/86 & 100/95 & 100/99\\
18 & Pairs, size & 12 & .20 & 100/100 & 100/100 & 100/100 & 100/100\\
19 & Triples & 12 & .20 & 100/100 & 100/100 & 100/100 & 100/100\\
20 & Quartets & 12 & .20 & 100/85 & 100/93 & 100/92 & 100/99\\
21 & Mixed & 12 & .20 & 100/98 & 100/100 & 100/99 & 100/100\\
22 & Pairs, two active & 16 & .20 & 97/37 & 100/81 & 100/94 & 100/98\\
23 & Quartets, stress & 20 & .05 & 100/0 & 100/0 & 100/0 & 100/0\\
24 & Quartets, stress & 20 & .20 & 100/0 & 100/0 & 100/0 & 100/0\\
\midrule
\multicolumn{4}{l}{Total (2,400 datasets per tolerance)} & 2095/1600 & 2184/1833 & 2250/1859 & 2303/1898\\
\bottomrule
\end{tabular*}
\end{table}

Table~\ref{tab:bma-confirmatory-all} reports feasible/certified counts out of 100 independent datasets per setting and tolerance. Of all 9,600 primary attempts, 8,832 establish feasibility, 7,241 attain the optimum, and 7,190 certify it; the 768 nonreturns remain included in Table~\ref{tab:bma-confirmatory-all}. The primary stopping outcomes comprise 7,190 certificates, 1,539 time limits, and 871 exhausted score budgets. In the quartet stress settings every primary output is feasible but none is optimal. At the $.90$ target their median complexity ratios are 2.98 and 3.13 for $b=.05$ and $.20$. Exact-mass region ranking, an oracle unavailable to practical search, attains the optimum in 77.3 percent of all 2,400 datasets at that target, with 90th-percentile complexity ratio 1.024.

\noindent\textit{Paired component comparisons.} At the $.90$ mass target, compare $H_N$ with $\widetilde H_N$, equal with adaptive allocation, and 50,000 requests/120 seconds or 100,000 requests/180 seconds with the primary budget. These comparisons use the first twenty datasets in cases 7, 8, 15, 17, 20, 22, 23, and 24. Forward, reverse, and fixed shuffled tree orders are compared on the first twenty datasets in cases 7, 8, 17, and 22.

\begin{table}[!tbp]
\centering
\caption{Paired representation-complexity reductions in the quartet stress settings.}
\label{tab:bma-confirmatory-paired}
\small
\setlength{\tabcolsep}{5pt}
\begin{tabular*}{\textwidth}{@{\extracolsep{\fill}}lcrrr@{}}
\toprule
Comparison & $b$ & \shortstack{Mean reduction\\(percent)} & \shortstack{95\% bootstrap\\interval} & \shortstack{Better/\\tied/worse}\\
\midrule
$\widetilde H_N$ versus $H_N$ & .05 & 16.3 & [13.3, 19.7] & 20/0/0\\
$\widetilde H_N$ versus $H_N$ & .20 & 9.8 & [7.3, 12.2] & 18/2/0\\
Adaptive versus equal allocation & .05 & 49.8 & [40.8, 58.3] & 19/0/1\\
Adaptive versus equal allocation & .20 & 37.3 & [24.8, 50.2] & 18/0/2\\
50,000 calls / 120 seconds & .05 & 40.8 & [35.2, 46.2] & 20/0/0\\
50,000 calls / 120 seconds & .20 & 38.2 & [32.6, 43.7] & 20/0/0\\
100,000 calls / 180 seconds & .05 & 54.6 & [50.0, 59.1] & 20/0/0\\
100,000 calls / 180 seconds & .20 & 49.9 & [43.9, 55.5] & 20/0/0\\
\bottomrule
\end{tabular*}
\end{table}

All 160 datasets in each interval/allocation/budget comparison give feasible unions. Relative to $H_N$, KL refinement lowers/ties/raises complexity on 38/121/1; relative to equal allocation, adaptive allocation does so on 38/119/3. Either larger budget reduces complexity on 44 datasets and leaves it unchanged on 116. Table~\ref{tab:bma-confirmatory-paired} isolates the forty difficult quartet datasets, with twenty per row. Positive reductions favor the first method or larger budget, and better/tied/worse counts compare attained complexities. Writing $C_0$ and $C_1$ for complexities before and after the stated change, the reductions are $(C_0-C_1)/C_0$, with descriptive percentile intervals from 5,000 dataset-bootstrap resamples per setting, without multiplicity adjustment. At the largest budget, eight of forty searches attain the optimum but none certifies it. Forward/reverse/shuffled orders yield feasibility on 80/77/80 datasets and certification on 76/37/74.

\noindent\textit{Why searches stop without certification.} Table~\ref{tab:bma-search-regimes} distinguishes limited exploration from imprecise exposed masses at the $.90$ mass target, with counts including every primary search in the stated cases. For all 200 quartet searches, median lower and upper complexity bounds divided by the audited optimum are $.50$ and $3.05$; neither bound equals the optimum in any case. Their full terminal partitions are exposed, so missing leaves cannot explain the gap.

\begin{table}[!tbp]
\centering
\caption{Search diagnostics for twenty-predictor designs.}
\label{tab:bma-search-regimes}
\small
\setlength{\tabcolsep}{6pt}
\begin{tabular*}{\textwidth}{@{\extracolsep{\fill}}lrr@{}}
\toprule
Diagnostic & Pairs (cases 10--12) & Quartets (cases 23--24)\\
\midrule
Independent datasets & 300 & 200\\
Supplied terminal regions & 59,049 & 243\\
Median exposed terminal regions & 303 & 243\\
Searches using stochastic refinement & 0/300 & 200/200\\
Feasible outputs & 150/300 & 200/200\\
Certified optima & 7/300 & 0/200\\
\bottomrule
\end{tabular*}
\end{table}

\noindent\textit{Selection comparison.} An exploratory selector comparison uses twelve datasets from each of cases 7, 8, 23, and 24. It replaces only conservative selection with ranking by estimated mass or mass per complexity. All variants retain ARS's optimistic problem, traversal, bounds, proposals, and adaptive allocation under 25,000 requests and 60 seconds. Estimates use direct sums, latest sampling means, or mass-scale interval midpoints. The first conservatively feasible ranked prefix is verified with directed rounding. All three methods are feasible on all 48 datasets. ARS certifies 22 optima and either ranking rule 13; every certificate occurs in the 24 deterministic paired cases. Mean paired complexity reductions are 2.14 and 2.23 percent, with lower/equal/higher counts 29/18/1 and 30/17/1. None attains or certifies the optimum in the 24 sampled quartet cases. The paired comparisons use the same 48 datasets as their ARS reference outcomes.

Component experiments use one to three concurrent workers, whereas the main comparison uses one; runtime is therefore summarized within each experiment, and time-limited outcomes reflect the corresponding computational load.

\subsection{Bayesian-network simulation}
\label{app:network-simulation}

The generating DAG has edges $1\to3$, $2\to3$, $2\to4$, $3\to5$, and $4\to5$. Let $P_j$ be the parent set of network vertex $j$. Vertices without parents have probability one half; the others follow
\[
  \operatorname{logit}\Pr(X_j=1\mid X_{P_j})
  =\frac{a\sum_{i\in P_j}(2X_i-1)}{\sqrt{\lvert P_j\rvert }}.
\] Crossing $a\in\{.7,1.2\}$ with $n\in\{100,300\}$ and 50 replications gives 200 independent datasets. The full five-vertex DAG posterior uses BDeu equivalent sample size one and pair probabilities $(.8,.1,.1)$ for absence and the two directions before conditioning on acyclicity.

Enumeration scores all 29,281 DAGs before search and supplies numerical regional sums with coincident endpoints. The binary tree resolves the ten unordered edges, giving 1,024 skeleton regions and 2,047 nodes, under lexicographic and reverse orders. Each region has complexity one plus its edge count. Four targets and both orders produce 1,600 searches. Posterior-ranked DAG and skeleton prefixes receive the same full mass information and the same complexity formula per listed object. The whole-family reference uses the same selector over all terminal masses, checking exploration and stopping rather than an independent optimization implementation.

Section~\ref{sec:network-simulation} reports the outcomes and aggregation conventions. Repeated tolerances and orders reuse datasets. The 10,000-request and 10,000-action allowances apply after enumeration and tree construction; saved search times exclude this preparation and do not assess end-to-end network scalability.

``Exact'' means full summation here, subject to floating-point arithmetic. This study evaluates representation and selection with enumerated masses, separately from stochastic regional estimation.

\subsection{Tecator specifications and sensitivities}
\label{app:tecator-ars}
\label{app:method-validation}

Table~\ref{tab:tecator-ars-all} collects the six sensitivity settings.

\begin{table}[!htb]
\centering
\caption{Tecator results across noise levels and mass targets.}
\label{tab:tecator-ars-all}
\label{tab:method-validation}
\small
\setlength{\tabcolsep}{6pt}
\begin{tabular*}{\textwidth}{@{\extracolsep{\fill}}ccrrrrr@{}}
\toprule
$\sigma$ & Target & \shortstack{Complexity\\bounds} & Optimum & \shortstack{Omitted-mass\\bound} & \shortstack{Retained\\mass} & \shortstack{Time\\(seconds)}\\
\midrule
.25 & .90 & $[43,43]$ & 43 & .095522 & .923539 & 33.9\\
.25 & .95 & $[44,44]$ & 44 & .049959 & .953188 & 36.9\\
.50 & .90 & $[59,67]$ & 59 & .096233 & .910131 & 25.9\\
.50 & .95 & $[92,101]$ & 93 & .047976 & .955360 & 26.5\\
1.00 & .90 & $[135,135]$ & 135 & .099900 & .904621 & 15.5\\
1.00 & .95 & $[177,177]$ & 177 & .048056 & .954179 & 16.7\\
\bottomrule
\end{tabular*}
\end{table}

\looseness=-1
Tecator spectra and fat measurements come from the \href{https://lib.stat.cmu.edu/datasets/tecator}{StatLib archive}, with its original attribution and redistribution note retained. Section~\ref{sec:tecator-ars} gives the training/test split, sixteen channels, groups, posterior, and budgets. Training correlations among the selected predictors range from $.9609$ to $.9999$. Training centering removes a flat common intercept, changing support weights by a common factor; predictions restore the training response mean. No test observations enter standardization or search. Changing $\sigma$ also changes the coefficient prior variance $\sigma^2\tau^2$, so each sensitivity defines a different conditional posterior.

\looseness=-1
All six combinations of noise level and mass target use safeguarded Newton envelope optimization and support-preserving proposal normalization, with mass-conservation propagation disabled. All use 50,000 requests and a 180-second allowance. Table~\ref{tab:tecator-ars-all} reports the mass and complexity bounds computed during search alongside the optimal complexities and retained masses from independent enumeration. Omitted-mass bounds are rounded upward; equal complexity endpoints certify optimality. Every search uses stochastic refinement and returns a union meeting its mass target. Four searches certify optimality; the two $\sigma=.5$ searches exhaust their score budgets with unresolved complexity gaps.

Independent enumeration of all 65,536 supports audits the search and supplies the inclusion probabilities and predictions in the main text.

\subsection{Sachs regional sums, search, and sensitivities}
\label{app:sachs-details}

\emph{Data and posterior.} The 853 observations and supplied discretization come from \href{https://www.bnlearn.com/book-crc-2ed/}{\emph{Bayesian Networks: With Examples in R}}, excluding interventions. Let $\mathcal N=\{1,\ldots,11\}$ index the measured variables in file order, and let $w_i(P)$ be the BDeu marginal likelihood for variable $i$ with parent set $P\subseteq\mathcal N\setminus\{i\}$, multiplied by $[(1-\alpha)/(2\alpha)]^{|P|}$. Here $\alpha$ is the prior edge-absence probability. The product of these local weights gives $s(\gamma)$ up to a graph-independent factor. All DAGs are allowed. The primary BDeu equivalent sample size is $\eta=1$, with $\alpha=.8$. Sensitivities change $\eta$ to 10 or $\alpha$ to $.6$ (Table~\ref{tab:sachs-sensitivity}).

\begin{table}[!tb]
\centering
\caption{Sachs sensitivities. Results are conditional on numerical regional sums.}
\label{tab:sachs-sensitivity}
\small
\setlength{\tabcolsep}{5pt}
\begin{tabular*}{\textwidth}{@{\extracolsep{\fill}}lcrrrr@{}}
\toprule
Specification & Target & Skeletons & Complexity & Mass & Adjacency error\\
\midrule
Primary & .95 & 1 & 9 & .988255 & .011582\\
 & .99 & 2 & 19 & .999834 & .000097\\
\addlinespace[1pt]
$\eta=10$ & .95 & 1 & 11 & .970925 & .010224\\
 & .99 & 3 & 33 & .990740 & .002110\\
\addlinespace[1pt]
$\alpha=.6$ & .95 & 1 & 9 & .969600 & .030299\\
 & .99 & 2 & 19 & .999896 & .000036\\
\bottomrule
\end{tabular*}
\par\vspace{3pt}
\parbox{\textwidth}{The primary specification has $\eta=1$ and $\alpha=.8$. Adjacency error is the largest absolute change over the 55 pairwise edge probabilities after conditioning on the retained union. Each 90 percent result equals its 95 percent row.}
\end{table}

\emph{Regional sums.} Fix the required and forbidden adjacencies of a partial or complete skeleton. For nonempty $T\subseteq S\subseteq\mathcal N$, let $F(S,T)$ sum compatible DAG weights on $S$ whose final longest-path layer is exactly $T$. A required adjacency within $T$ gives $F(S,T)=0$. Otherwise, $F(T,T)=\prod_{i\in T}w_i(\varnothing)$ and, for $B=S\setminus T\ne\varnothing$,
\[
 F(S,T)=\sum_{\varnothing\ne V\subseteq B}F(B,V)
 \prod_{i\in T}\biggl(\sum_{P\subseteq B:\,P\cap V\ne\varnothing}w_i(P)\biggr).
 \tag{C.1}\label{eq:sachs-regional-sum}
\]
Parent sets exclude forbidden neighbors and include every required neighbor in $B$. Deleting $T$ leaves a unique earlier DAG with final layer $V$; each added vertex has a parent in $V$. This unique layering counts each DAG once. Summing $F(\mathcal N,T)$ over nonempty $T$ gives the constrained $Z_v$, or $Z_r$ without constraints. Computation uses log-space arithmetic and positive parent sums when subtraction would lose precision.

\emph{Search and selection bounds.} The binary tree resolves pairs lexicographically, with node labels recording edge presence or absence. Search expands the largest-mass region, breaking ties by depth and then edge masks. Discovery stops when completed-leaf mass exceeds $.99+10^{-8}$ or a limit of 2,000 cached regions or 3,600 seconds is reached. Remaining completed skeletons are then admitted. For each target, HiGHS minimizes complexity among discovered skeleton unions, with a 30-second limit and zero requested relative gap. All reported solves return optimal status.

For the final frontier $\mathcal F^{(k)}$, including its completed leaves $\mathcal E^{(k)}$, write $p_v=Z_v/Z_r$ and let $e_v$ count adjacencies already required at node $v$. The smallest descendant complexity is $\underline c_v=1+e_v$, equal to $c_v$ at a leaf. Relaxing each frontier node to a fractional decision gives
\[
 L^{(k)}=\min_{x\in[0,1]^{\mathcal F^{(k)}}}
 \left\{\sum_{v\in\mathcal F^{(k)}}\underline c_v x_v:
 \sum_{v\in\mathcal F^{(k)}}p_vx_v\ge1-\epsilon\right\}.
 \tag{C.2}\label{eq:sachs-lower}
\]
Fractions are filled in decreasing $p_v/\underline c_v$ order. Replacing selected descendants by their entire frontier region increases mass at no greater complexity, making $L^{(k)}$ a lower bound. Every eligible leaf with $p_v>\epsilon$ is mandatory. The reported bound is the larger of $L^{(k)}$ and the sum of mandatory leaf complexities. Numerically, membership uses $p_v>\epsilon+10^{-10}$; rounding the bound upward after subtracting $10^{-9}$ certifies integer-complexity optimality when it reaches the returned complexity.

\emph{Results and numerical scope.} Ranking skeletons matches optimized complexity in every setting. For $\eta=10$ at 99 percent, the lower bound $32.0923$ establishes integer complexity 33 conditional on numerical masses. Other lower bounds equal returned complexities.

Checks recover DAG counts on two through five vertices and match eight weighted partial-skeleton sums against enumeration. Sachs child/parent and orientation/skeleton sums agree within $10^{-8}$. Orientation enumeration caps candidates at $10^6$ per skeleton; none is skipped. At the primary 95 percent target, the smallest retained DAG probability exceeds every unresolved region mass, verifying global ranking. Feasibility is checked within $10^{-8}$ using the full normalizer, including undiscovered DAGs. These checks do not enclose roundoff errors or invoke sampling coverage. The supplement preserves this study in \texttt{historical\_real\_data/real\_data/}. Cache reuse precludes cold-start timing comparisons.

\end{document}